\documentclass[twoside]{article}
\usepackage[accepted]{aistats2018_arxiv}
\usepackage{xr}
\usepackage{url}
\usepackage{hyperref}     
\usepackage{multirow}
\usepackage{amsmath}
\usepackage{amssymb}
\usepackage{amsthm}
\usepackage{amsfonts}
\usepackage[dvips]{graphicx}
\usepackage{color}
\usepackage{here}
\usepackage{wrapfig}
\usepackage{centernot}
\usepackage{algorithm}
\usepackage{algorithmic}

\newcommand{\ackname}{Acknowledgements}
\newtheorem{theorem}{Theorem}

\newtheorem{proposition}{Proposition}
\newtheorem{definition}{Definition}

\def\pd<#1>{\langle #1 \rangle}
\def\defeq{\overset{\mathrm{def}}{=}}

\begin{document}

%

%
\runningauthor{Atsushi Nitanda, Taiji Suzuki}

\twocolumn[
\aistatstitle{Gradient Layer: Enhancing the Convergence of Adversarial Training for Generative Models}

\aistatsauthor{Atsushi Nitanda$^{\dag,1,2}$  
  \And Taiji Suzuki$^{\ddag,1,2,3}$  }

\aistatsaddress{
  $^1$Graduate School of Information Science and Technology, The University of Tokyo\\
  $^2$Center for Advanced Intelligence Project, RIKEN \\
  $^3$PRESTO, Japan Science and Technology Agency }
]

\begin{abstract}
We propose a new technique that boosts the convergence of training generative adversarial networks.
Generally, the rate of training deep models reduces severely after multiple iterations.
A key reason for this phenomenon is that a deep network is expressed using a 
highly non-convex finite-dimensional model, and thus
the parameter gets stuck in a local optimum.
Because of this, methods often suffer not only from degeneration of the convergence speed but also from limitations in the representational power of the trained network.
To overcome this issue, we propose an additional layer called the {\it gradient layer}
to seek a descent direction in an {\it infinite-dimensional space}.  
Because the layer is constructed in the infinite-dimensional space, 
we are not restricted by the specific model structure of finite-dimensional models. 
As a result, we can get out of the local optima in finite-dimensional models and move towards the global optimal function more directly.
In this paper, this phenomenon is explained from the functional gradient method perspective of the gradient layer.
Interestingly, the optimization procedure using the gradient layer naturally constructs the deep structure of the network.
Moreover, we demonstrate that this procedure can be regarded as a discretization method of the gradient flow that naturally reduces the objective function.
Finally, the method is tested using several numerical experiments, which show its fast convergence. 
\end{abstract}

\section{Introduction}
Generative adversarial networks (GANs) \cite{GAN2014} are a promising scheme for learning generative models.
GANs are trained by a discriminator and a generator in an adversarial way.
Discriminators are trained to classify between real samples and fake samples drawn from generators, whereas generators are trained to mimic real samples.
Although training GANs is quite difficult, adversarial learning succeeded in generating very impressive samples \cite{radford2016unsupervised}, 
and there are many subsequent studies \cite{larsen2016autoencoding,salimans2016improved,NIPS2016_6066,NIPS2016_6399,zhang2017stackgan}.
Wasserstein GANs (WGANs) \cite{arjovsky2017wasserstein} are a variant to remedy the mode collapse that appears in the standard GANs by using 
the Wasserstein distance \cite{villani2008optimal}, although they also sometimes generate low-quality samples or fail to converge.
Moreover, an improved variant of WGANs was also proposed \cite{gulrajani2017improved} and it succeeded in generating high-quality samples and stabilizing WGANs.
Although these attempts have provided better results, there is still scope to improve the performance of GANs further.

One reason for this difficulty stems from the limitation of the representational power of the generator.
If the discriminator is optimized for the generator,
the behavior is solely determined by the samples produced from that generator.
In other words, for a generator with a poor representational power,
the discriminator terminates its learning in the early stage and consequently results in having low discriminative power.
However, for a finite-dimensional parameterized generator, 
the ability to generate novel samples to cheat the discriminators is limited. 
In addition, the highly non-convex structure of the deep neural network for the generator prevents 
us from finding a direction for improvement.
As a result, the trained parameter gets stuck in a local optimum and the training procedure does not proceed any more.

In this study, we propose a new learning procedure to overcome the issues of 
limited representational power and local optimum
by introducing a new type of layer called a {\it gradient layer}.
The gradient layer finds a direction for improvement in an infinite-dimensional space by computing the {\it functional gradient} \cite{luenberger1969optimization}
instead of the ordinary gradient induced by a finite-dimensional model. 
Because the functional gradient used for the gradient layer is not limited in the tangent space of a finite-dimensional model, 
it has much more freedom than the ordinary finite-dimensional one.
Thanks to this property, our method can break the limit of the local optimum induced by the strong non-convexity of a finite-dimensional model, which
gives much more representational power to the generator. 
We theoretically justify this phenomenon from the functional gradient method perspective and rigorously present a convergence analysis.
Interestingly, one iteration of the method can be recognized as inserting one layer into the generator and the total number of iterations is the number of inserted layers.
Therefore, our learning procedure naturally constructs the deep neural network architecture by inserting gradient layers.
Although, gradient layers can be inserted into an arbitrary layer, they are typically stacked on top of the generator in the final training phase to improve the generated sample quality.

Moreover, we provide another interesting perspective of the gradient layer, i.e., discretization of the gradient flow in the space of probability measures.
In Euclidean space, the steepest descent which is the typical optimization method, can be derived by discretizing the gradient flow that naturally produces a curve to reduce the objective function.
Because the goal of GANs is to generate a sequence of probability measures moving to the empirical distribution by training samples,
it is natural to consider a gradient flow in the space of probability measures defined by a distance between generated distribution and the empirical distribution
and to discretize it in order to construct practical algorithms.
We show that the functional gradient method for optimizing the generator in the function space is such a discretization method; in other words,
the gradient flow can be tracked by stacking gradient layers successively.

The recently proposed SteinGAN \cite{wang2016learning} is closely related to our work and has a similar flavor, but it is based on another strategy to track gradient flow.
That is, since that discretization is mimicked by a fixed-size deep neural network in SteinGAN, it may have the same limitation as typical GANs.
By contrast, our method directly tracks the gradient flow in the final phase of training GANs to break the limit of the finite-dimensional generator.

\begin{figure*}[t]
  \begin{center} 
          \includegraphics[width=125mm,angle=0]{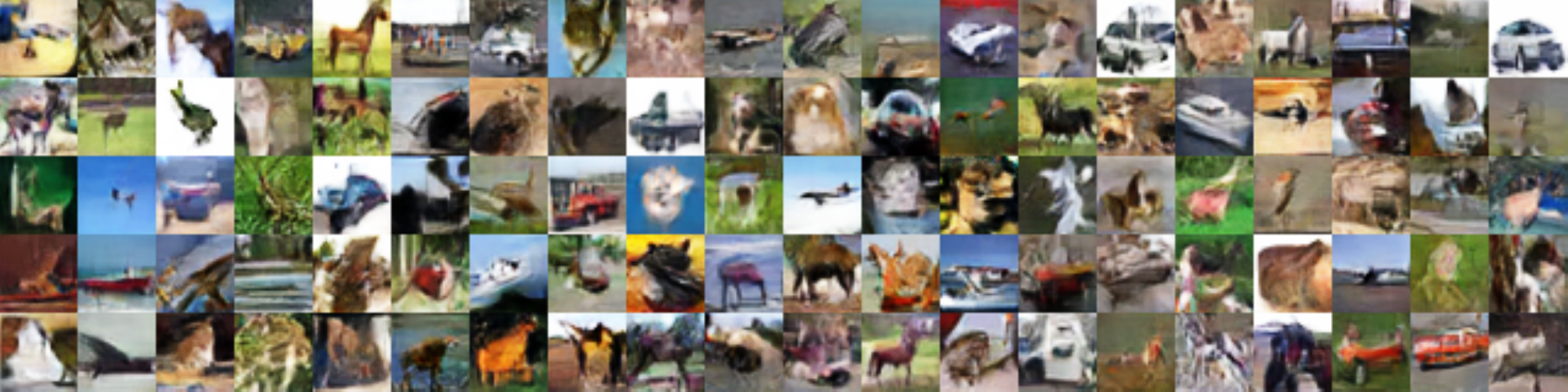} 
    \caption{Random samples drawn from the generator trained by Algorithm \ref{WGAN_FT} on the CIFAR-10 dataset.} \label{wgan_gl_sample}
\end{center}
\end{figure*}

\section{Brief Review of Wasserstein GANs}
In this section, we introduce WGANs and their variants.
Although our proposed gradient layer is applicable to various models, we demonstrate how it performs well for the training of generative models;
in particular, we treat Wasserstein GANs as a main application in this paper. 
Let us start from briefly reviewing WGANs.

WGAN is a powerful generative model based on the $1$-Wasserstein distance, defined as the $L^1$ minimum cost of transporting one probability distribution to the other.
Let $\mathcal{X}\subset\mathbb{R}^v$ and $\mathcal{Z}\subset\mathbb{R}^h$ be a compact convex data space and a hidden space, respectively.
A typical example of $\mathcal{X}$ is the image space $[0,1]^v$.
For a noise distribution $\mu_n$ on $\mathcal{Z}$, WGAN learns a data generator $g: \mathcal{Z}\rightarrow \mathcal{X}$
to minimize an approximation to the $1$-Wasserstein distance between the data distribution $\mu_D$ and the push-forward distribution $g_\sharp \mu_n$,
which is a distribution that the random variable $g(z)$ follows when $z\sim \mu_n$ (in other words, the distribution obtained by applying a coordinate transform $g$ to $z \sim \mu_n$).
That is, WGAN can be described as the following $\min\max$ problem by using a Kantrovich-Rubinstein duality form of the $1$-Wasserstein distance:
\[ \min_{g\in\mathcal{G}} \max_{f\in\mathcal{F}} \mathcal{L}(f,g) \defeq \mathbb{E}_{x\sim\mu_D}[f(x)]-\mathbb{E}_{z\sim\mu_n}[f\circ g(z)], \]
where $\mathcal{G}$ is the set of generators and $\mathcal{F}$ is an approximate set to the set of $1$-Lipschitz continuous functions called {\it critic}.
In WGANs, $\mathcal{G}, \mathcal{F}$ are parameterized neural networks $\{g_{\theta}\}, \{f_{\tau}\}$ and the problem is solved by alternate optimization: maximizing and minimizing $\mathcal{L}(f_\tau,g_\theta)$ with respect to $\tau$ and $\theta$, alternately.

In practice, to impose the Lipschitz continuity on critics $f_{\tau}$, penalization techniques were explored.
For instance, the original WGANs \cite{arjovsky2017wasserstein} use weight clipping $\|\tau\|_\infty \leq c$, which implies the upper-bound on the norm of $\nabla_\tau f_\tau$ and makes it Lipschitz continuous.
However, it was pointed out in a subsequent study \cite{gulrajani2017improved} that such a restriction seems to be unnatural and sometimes leads to a low-quality generator or a failure to converge.
In the same study, an improved variant of WGANs called WGAN-GP was proposed, which succeeded in stabilizing the optimization process and generating high-quality samples.
WGAN-GP \cite{gulrajani2017improved} adds the gradient penalty $( \|\nabla_{\tilde{x}}f_\tau(\tilde{x}) \|_2 - 1 )^2$ to the objective function in the training phase of critics,
where $\tilde{x}$ is a random interpolation between a training example $x \sim \mu_D$ and a generated sample $g(z) \sim g_{\theta\sharp}\mu_n$,
i.e., $\tilde{x} \leftarrow \epsilon x + (1-\epsilon)g(z)$ ($\epsilon \sim U[0,1]$: uniform distribution).
DRAGAN \cite{kodali2017train} is a similar method to WGAN-GP, although it is based on a different motivation.
DRAGAN also uses the gradient penalty, but the penalty is imposed on a neighborhood of the data manifold by a random perturbation of a training example.

WGAN and its variants are learned by alternately optimizing $f_\tau$ and $g_\theta$, as stated above.
We can regard this learning procedure as a problem of minimizing $\mathcal{L}(g_\theta) \defeq \max_{\tau}\{ \mathcal{L}(f_\tau, g_\theta) - \lambda R_{\tau}\}$, where $R_{\tau}$ is a penalty term. 
Let $\mathcal{L}(f_\tau,g_\theta)-\lambda R_{\tau}$ attain its maximum value at $\tau_*$ for $g_{\theta}$.
Then, the gradient $\nabla_\theta \mathcal{L}(g_\theta)$ is the same as $- \mathbb{E}_{\mu_n}[\nabla_\tau f_{\tau_*}^\top \nabla_\theta g_\theta(z)]$ by the envelope theorem \cite{milgrom2002envelope} 
when both terms are well-defined.
The differentiability of $\mathcal{L}(g_\theta)$ with respect to $\theta$ almost everywhere is proved in \cite{arjovsky2017wasserstein} under a reasonable assumption.
Hence, we can apply the gradient method to this problem by approximating this gradient with finite particles generated from $\mu_n$.
However, because it is difficult to obtain $f_{\tau_*}$, we run the gradient method for several iterations on training a critic instead of exactly computing $f_{\tau_*}$ at each $g_\theta$.
We can notice that this learning procedure is quite similar to that of the standard GAN \cite{GAN2014}.



\section{Gradient Layer}
In the usual training procedure of WGANs, though more general maps are admissible for the original purpose,
generators are parameterized by finite-dimensional space as described in the previous section, and 
the parameter may get stuck in a local optimum induced by this restriction, or the speed of convergence may reduce.
In this work, we propose a gradient layer that accelerates the convergence and breaks the limit of finite-dimensional models.
This layer is theoretically derived by the infinite-dimensional optimization method.
We first explain the high-level idea of the gradient layer that strictly improves the ability of generator
and why our method enhances the convergence of training WGANs.

\subsection{High-level idea of gradient layer}
Here, we explain gradient layer with intuitive motivation.
It is inserted into the generator $g$ in WGANs.
We now focus on minimizing $\mathcal{L}(f,g)$ with respect to $g$ under a fixed critic $f$,
that is, we consider the problem $\min_g \mathcal{L}_f(g) \defeq \mathbb{E}_{\mu_n}[-f(g(z))]$.
Let us split $g$ into two neural networks $g=g_1\circ g_2$ at arbitrary layer where a new layer is to be inserted.
Our purpose is to specify the form of layer $\phi$ that reduces the objective value by perturbations of inputs $g_2(z)$, i.e., $\mathcal{L}_f(g_1\circ \phi \circ g_2)\leq \mathcal{L}_f(g)$.
Since $\mathcal{L}_f(g_1 \circ \phi \circ g_2)$ is regarded as the integral $\mathbb{E}_{z'\sim g_{2\sharp}\mu_n}[-f(g_1(\phi(z')))]$ with respect to the push-forward distribution $g_{2\sharp}\mu_n$,
this purpose is achieved by transporting the input distribution of $\phi$ along the gradient field $\nabla_{z'} f(g_1(z'))$.
Therefore, we propose a gradient layer $G_{\eta}$ with one hyperparameter $\eta > 0$ 
as a map that transforms an input $z'$ to

\begin{equation}
  G_{\eta}(z') = z' + \eta \nabla_{z'} f(g_1(z')). \label{grad_layer}
\end{equation}
Because the gradient layer depends on the parameters $\tau, \theta$ of the upper layers $f, g_1$, we specify the parameter as $G_\eta^{\tau,\theta}$ if needed.

Applying the gradient layer recursively, it further progresses and achieves a better objective. 
The computation of the gradient layer is quite simple. Actually, simply taking the derivative is sufficient, which can be efficiently executed.
Because too many gradient layers would lead to overfitting to the critic $f$, we stop stacking the gradient layer after an appropriate number of steps.
Indeed, if $f\circ g_1$ is Lipschitz continuous, $id+\eta f\circ g_1$ for sufficiently small $\eta$ is an injection because $(id+\eta f\circ g_1)(z)=(id+\eta f\circ g_1)(z')$ implies
$\|z-z'\|_2 \leq \eta \mathcal{L}_{f\circ g_1}\| z-z'\|_2 $ where $\mathcal{L}_{f\circ g_1}$ is the Lipschitz constant.
Thus, a topology of ${\rm supp}(g_{2\sharp}\mu_n)$ is preserved and early stopping is justified.
Then, this layer efficiently generates high-quality samples for the critic and the overall adversarial training procedure can be also boosted.

\subsection{Powerful optimization ability}
Because the gradient layer directly transports inputs as stated above, it strictly improves the objective value if there is room for optimization,
unlike finite-dimensional models that may be trapped in local optima induced by the restriction of generators.
Indeed, when the gradient layer cannot move inputs, i.e., $G_\eta(z')=z'$, the gradient $\nabla_{z'} f(g_1(z'))$ vanishes on ${\rm supp}(g_{2\sharp}\mu_n)$
and there is no chance to improve the objective value by optimizing $g_2$ because of the chain rule of derivatives.
We now explain this phenomenon more precisely.
Let us first consider the training of $g_2$ in the usual way.
We denote by $w_2$ the parameter of $g_2$.
As stated in the previous section, $w_2$ is updated by using the gradient
\begin{equation}
  \mathbb{E}_{\mu_n}[ J_{w_2}^\top g_2(z) \nabla_{z'}f(g_1(g_2(z))) ], \label{finite_grad}
\end{equation}
where $z'$ is the input to $g_1$ and $J_{w_2}g_2(z)$ is the Jacobian matrix of $g_2$ with respect to $w_2$.
We immediately notice that this gradient (\ref{finite_grad}) is the inner product of $J_{w_2}^\top g_2(\cdot)$ and $\nabla_{z'}f(g_1(g_2(\cdot)))$ in $L^2(\mu_n)$-space,
and the latter term is the output of the gradient layer.
Therefore, when the gradient layer cannot move the inputs, the gradient with respect to $w_2$ also vanishes.
However, even if the gradient vanishes, the gradient layer $G_\eta$ can move the inputs in general.
Thus, whereas the optimization of $g_2$ may get stuck in a local optimum or be slowed down in this case, the gradient layer strictly improves the quality of the generated samples for the upper layers,
because $\nabla_{z'}f(g_1(g_2(z)))$ does not vanish.
This is the reason the gradient layer has a greater optimization ability than finite-dimensional models.

\subsection{Algorithm description}
The overall algorithm is described in this subsection.
We adopt WGAN-GP as the base model to which gradient layer is applied.
Let us denote by $R_{f_\tau}(\tilde{x})$ a gradient penalty term.
In a paper on the improved WGANs \cite{gulrajani2017improved}, the use of a two-sided penalty $( \|\nabla_{\tilde{x}}f_\tau(\tilde{x}) \|_2 - 1 )^2$ is recommended.
However, we also allow the use of the one-sided variant $(\max(\|\nabla_{\tilde{x}}f_\tau(\tilde{x})\|_2-1,0))^2$.
As for the place in which the gradient layer is inserted, we can propose several possibilities, e.g.,
inserting the gradient layer into (i) the top and  (ii) the bottom of the layers of the generator.
The latter usage is described in the appendix.

The first usage is stacking gradient layers on the top of the generator, except for normalization to fine-tune the generator in the final phase.
Although a normalization term such as ${\rm tanh}$ is commonly stacked on generators to bound the output range of the generators, gradient layers are typically applied before the normalization layer.
Since ${\rm tanh}$ is a fixed function, it is no problem to combine ${\rm tanh}$ with critics by reinterpreting $\mathcal{F}$ and $\mathcal{X}$.
The gradient layer directly handles the generated samples, so that it may significantly improve the sample quality.
Because the gradient $\nabla_x f_\tau(x)$ of the critic with respect to data variables provides the direction to improve the quality of the current generated samples,
it is expected that we can obtain better results by tracking the gradient iteratively. 
To compute the output from the gradient layer for a completely new input, 
we need to reproduce the computation of the gradient layers, which can be 
realized by saving the history of the parameters of critics and stacking the gradient layers using these parameters. The concrete procedure is described in Algorithm \ref{WGAN_FT}.
When executing Algorithm \ref{WGAN_FT}, the parameter of $g_\theta$ is fixed, so that the push-forward measure $g_{\theta\sharp}\mu_n$ is treated as a base probability measure and we denote it by $\mu_g$.
Because the gradient layers depend on the history of the parameters in this case, we specify the parameter to be used: $G_\eta^\tau$.
For the parameter $\tau$ and the gradient $v$, we denote by $\mathcal{A}(\tau,v)$ one step of a gradient-based method such as SGD with momentum,
Adam \cite{kingma2015adam}, and RMSPROP \cite{tieleman2012lecture}.
From the optimization perspective, we show that Algorithm \ref{WGAN_FT} can be regarded as an approximation to the functional gradient method.
From this perspective, we show fast convergence of the method under appropriate assumptions where the objective function is smooth and the critics are optimized in each loop.
This theoretical justification is described later.
Although Algorithm \ref{WGAN_FT} has a great optimization ability, applying the algorithm to large models is difficult because it requires the memory to register 
parameters; thus, we propose its usage for fine-tuning in the final phase of training a WGAN-GP.
After the execution of Algorithm \ref{WGAN_FT}, we can generate samples by using the history of critics, the learning rate, and the base distribution as described in Algorithm \ref{data_gen_proc}.

\begin{algorithm}[h]
  \caption{Finetuning WGAN-GP}
  \label{WGAN_FT}
\begin{algorithmic}
  \STATE {\bfseries Input:} The base distribution $\mu_g=g_\sharp \mu_n$, the minibatch size $b$, the number of iterations $T$,
  the initial parameters $\tau_0$ of the critic, 
  the number of iterations $T_0$ for the critic, the regularization parameter $\lambda$, and the learning rate $\eta$ for gradient layers.\\
   \vspace{1mm}

   \FOR{$k=0$ {\bfseries to} $T-1$}
   \STATE $\tau \leftarrow \tau_k$\\
   \FOR{$k_0=0$ {\bfseries to} $T_0-1$}
   \STATE$\{x_i\}_{i=1}^b \sim \mu_D^b$, $\{z_i\}_{i=1}^b \sim \mu_g^b$, $\{\epsilon_i\}_{i=1}^b \sim U[0,1]^b$ \\
   \STATE $\{z_i\}_{i=1}^b \leftarrow \{G^{\tau_k}_\eta\circ \cdots \circ G^{\tau_1}_\eta(z_i)\}_{i=1}^b$
   \STATE $\{\tilde{x}_i\}_{i=1}^b \leftarrow \{\epsilon_i x_i + (1-\epsilon_i)z_i\}_{i=1}^b$
   \STATE $v \leftarrow \nabla_\tau \frac{1}{b}\sum_{i=1}^b [f_\tau(z_i)-f_\tau(x_i) + \lambda R_{f_\tau}(\tilde{x}_i)])$
   \STATE $\tau \leftarrow \mathcal{A}(\tau,v)$ \\
   \ENDFOR
   \STATE $\tau_{k+1}\leftarrow \tau$\\   
   \ENDFOR
   \STATE Return $(\tau_1,\ldots,\tau_T)$.
\end{algorithmic}
\end{algorithm}

\begin{algorithm}[h]
  \caption{Data Generation for Algorithm \ref{WGAN_FT}}
  \label{data_gen_proc}
\begin{algorithmic}
   \STATE {\bfseries Input:} the seed drawn from base measure $z\sim \mu_g=g_\sharp \mu_n$, the history of parameters $\{\tau_k\}_{k=1}^T$, and the learning rate $\eta$ for gradient layers.\\
   \vspace{1mm}
   \STATE Return the sample $G^{\tau_T}_\eta\circ \cdots \circ G^{\tau_1}_\eta(z)$.
\end{algorithmic}
\end{algorithm}

\section{Functional Gradient Method} \label{sec:funct_grad}
In this section, we provide mathematically rigorous derivation from the functional gradient method \cite{luenberger1969optimization} perspective
under the Fr\'{e}chet differentiable (functional differentiable) assumption on $\mathcal{L}$.
That is, we consider an optimization problem with respect to a generator in an infinite-dimensional space.
For simplicity, we focus on the case where the gradient layer is stacked on top of a generator $g$ and we treat $g_\sharp \mu_n$ as the base measure $\mu_g$.
Thus, in the following we omit the notation $g$ in $\mathcal{L}(f,\phi \circ g)$.
Let $L^2(\mu_g)$ be the space of $L^2(\mu_g)$-integrable maps from $\mathbb{R}^v$ to $\mathbb{R}^v$, equipped with the $\pd<\cdot,\cdot>_{L^2(\mu_g)}$-inner product:
for $\forall \phi_1, \forall \phi_2 \in L^2(\mu_g)$,
\[ \pd<\phi_1,\phi_2>_{L^2(\mu_g)} = \mathbb{E}_{\mu_g}[\phi_1(z)\top \phi_2(z)]. \]
To learn WGAN-GP, we consider the infinite-dimensional problem:
\begin{align*}
  \min_{\phi \in L^2(\mu_g)} \max_{f_\tau\in\mathcal{F}} \mathcal{L}(f_\tau,\phi)-\lambda R_{f_\tau},
\end{align*}
where $R_{f_\tau}$ is a gradient penalty term.
To achieve this goal, we take a G\^{a}teaux derivative along a given map $v \in L^2(\mu_g)$, i.e., a directional derivative along $v$.
Let us denote $\max_{f_{\tau}\in \mathcal{F}}\{\mathcal{L}(f_{\tau},\phi)-\lambda R_f\}$ by $\mathcal{L}(\phi)$ and
$\arg \max_{f_{\tau}\in \mathcal{F}}\{\mathcal{L}(f_{\tau},\phi)-\lambda R_{f_{\tau}}\}$ by $f^*_\phi$ and the corresponding parameter by $\tau_\phi^*$, i.e.,
$f^*_\phi = f_{\tau_\phi}^*$ for $\phi \in L^2(\mu_g)$.
If every $f \in \mathcal{F}$ is Lipschitz continuous and differentiable, we can find that by the envelope theorem and Lebesgue's convergence theorem 
this derivative takes the form:
\begin{equation}
  \frac{d}{dt}\mathcal{L}(\phi+tv)\Big|_{t=0} = - \mathbb{E}_{\mu_g}[ \nabla_{x}f_\phi^*(x)|_{x=\phi(z)}^\top v(z)]. \notag
\end{equation}
Therefore, $-\nabla_{x}f^*_{\phi}(x)|_{x={\phi}(\cdot)}$ can be regarded as a Fr\'{e}chet derivative (functional gradient) in $L^2(\mu_g)$
and we denote it by $\nabla_\phi \mathcal{L}(\phi)$, which performs like the usual gradient in Euclidean space.
Using this notation, the optimization of $\mathcal{L}(\phi)$ can be accomplished by Algorithm \ref{FGD}, which is a gradient descent method in a function space.
Because the functional gradient has the form $-\nabla_x f_\phi^* \circ \phi$, each iteration of the functional gradient method with respect to $\phi$ is
$\phi \leftarrow \phi + \eta \nabla_x f_\phi^* \circ \phi=(id+\eta \nabla_x f_\phi^*)\circ \phi$, where $\eta$ is the learning rate.
We notice here that this iteration is the composition of a perturbation map $id+\eta \nabla_xf_\phi^*$ and a current map $\phi$
and is nothing but stacking a gradient layer $G_\eta^{\tau_*}$ on $\phi(z)$.
In other words, the functional gradient method with respect to $\phi$, i.e., Algorithm \ref{FGD}, is the procedure of building a deep neural network by inserting gradient layers, where
the total number of iterations is the number of layers.
Moreover, we notice that if we view $\nabla_x f_\phi^{*}$ as a perturbation term, this layer resembles that of {\it residual networks} \cite{he2016deep}
which is one of the state-of-the-art architectures in supervised learning tasks.

However, executing Algorithm \ref{FGD} is difficult in practice because the exact optimization with respect to a critic $f$ to compute $\mathcal{L}(\phi)$ is a hard problem.
Thus, we need an approximation and we argue that Algorithm \ref{WGAN_FT} is such a method.
This point can be understood as follows.
Roughly speaking, it maximizes $\mathcal{L}(f,\phi)$ with respect to $f$ in the inner loop under fixed $\phi=G_\eta^{\tau_{k}} \circ G_\eta^{\tau_{k-1}} \circ \dots \circ G_\eta^{\tau_{1}}$
to obtain an approximate solution $\tau_{k+1}$ to $\tau_*$ and minimizes that with respect to $\phi$ in the outer loop by stacking $G_\eta^{\tau_{k+1}}$,
which is an approximation to $G_\eta^{\tau_*}$.
Thus, Algorithm \ref{WGAN_FT} is an approximated method, but we expect it to achieve fast convergence owing to the powerful optimization ability of the functional gradient method, as shown later.
In particular, it is more effective to apply the algorithm in the final phase of training WGAN-GP to fine-tune it, because the optimization ability of parametric models are limited.

\begin{algorithm}[h]
  \caption{Functional Gradient Descent}
  \label{FGD}
\begin{algorithmic}
   \STATE {\bfseries Input:} the initial generator $g$ and the learning rate $\eta$.\\
   \vspace{1mm}
   $\phi_0 \leftarrow g$
   \FOR{$k=0$ {\bfseries to} $T-1$}
   \STATE $\phi_{k+1} \leftarrow \phi_k - \eta \nabla_\phi \mathcal{L}(\phi_k)$ \\
   \ENDFOR
   \STATE Return the function: $\phi_T$.
\end{algorithmic}
\end{algorithm}

\section{Convergence Analysis}
Let us provide convergence analysis of Algorithm \ref{FGD} for the problem of the general form: $\min_\phi \mathcal{L}(\phi)$.
The convergence can be shown in an analogous way to that for the finite-dimensional one.
To prove this, we make a smoothness assumption on the loss function.
We now describe a definition of the smoothness on a Hilbert space whose counterpart in finite-dimensional space is often assumed for smooth non-convex optimization methods.

\begin{definition} \label{smooth_definition}
  Let $h$ be a function on a Hilbert space $(\mathcal{Z},\pd<,>_{\mathcal{Z}})$.
  We call that $h$ is $L$-smooth at $z$ in $U$ if $h$ is differentiable at $z$ and it follows that $\forall z' \in U$.
  \begin{equation*}
    | h(z') - h(z) - \pd<\nabla_z h(z),z'-z>_{\mathcal{Z}} | \leq \frac{L}{2}\|z'-z\|^2_{\mathcal{Z}}.
  \end{equation*}
\end{definition}
The following definition and proposition provide one condition leading to Lipschitz smoothness of $\mathcal{L}$.
Let us denote by $\|\cdot\|_{L^\infty(\mu_g)}$ the sup-norm $\|\psi\|_{L^\infty(\mu_g)}=\sup_{{\rm supp(\mu_g)}}\|\psi(z)\|_2$ and by $B_r^{\infty}(\phi)$ a ball of center $\phi$ and radius $r$.
Let $\hat{\mathcal{L}}(f,\psi) = \mathcal{L}(f,\psi) - \lambda R_{f}$.
In the following we assume $f_{\psi}^*$ is uniquely defined for $\psi\in L^2(\mu_g)$ and $L$-smoothness with respect to the input $x$.
\begin{definition} \label{regular_definition}
  For positive values $r$ and $L$, we call that $\mathcal{L}$ is $(r,L)$-regular at $\phi$ when the following condition is satisfied; For $\forall \psi \in B_r^{\infty}(\phi)$,
  $\hat{\mathcal{L}}(f_{\psi'}^*,\psi)$ is $L$-smooth at $\psi$ with respect to $\psi'$ in $B_r^{\infty}(\psi)$.
\end{definition}

\begin{proposition} \label{surrogate_prop}
  If $\mathcal{L}$ is $(r,L)$-regular at $\phi$, then $\mathcal{L}$ is $2L$-smooth at $\phi$ in $B_r^{\infty}(\psi)$.
\end{proposition}

We now show the convergence of Algorithm \ref{FGD}.
The following theorem gives the rate to converge to the stationary point.

\begin{theorem} \label{convergence_theorem}
  Let us assume the norm of the gradient $\|\nabla_x f_\phi^*(x)\|_2$ is uniformly bounded by $\alpha$
  and assume $\mathcal{L}$ is $L$-smooth at $\phi$ in $B_r^\infty(\phi)$ for $\forall \phi \in L^2(\mu_g)$.
  Suppose we run Algorithm \ref{FGD} with constant learning rate $\eta \leq \min\{1/L,r/\alpha\}$.
  Then we have for $T\in \mathbb{Z}_+$
  \begin{equation*}
    \min_{k \in \{0,\ldots,T-1\}}\|\nabla_\phi \mathcal{L}(\phi_k)\|^2_{L^2(\mu_g)} \leq \frac{2}{\eta T}(\mathcal{L}(\phi_0) - \mathcal{L}_*),
  \end{equation*}
  where $\mathcal{L}_*= \inf_\phi \mathcal{L}(\phi)$.
\end{theorem}

Note that the convergence rate $O(1/T)$ is the same as the gradient descent method for smooth objective in the finite-dimensional one.
This means that even though the optimization is executed in the infinite-dimensional space,
we do not suffer from the infinite dimensionality in terms of the convergence.

The following rough argument indicates that Algorithm \ref{FGD} matches with learning WGANs.
Let $W_1$ denote the $1$-Wasserstein distance with respect to the Euclidean distance on a compact base space $\mathcal{X} \subset \mathbb{R}^v$.
The following proposition is immediately shown by combining existing results \cite{ambrosio2003lecture,sudakov1979geometric}.

\begin{proposition} \label{opt_wasserstein_prop}
  Let $\mu_g$ be a Borel probability measure on $\mathcal{X}$ and assume $\mu_g$ is absolutely continuous with respect to the Lebesgue measure.
  Then, there exists an optimal transport $\psi$ and it follows that $W_1(\psi_{t\sharp}\mu_g, \mu_D) = (1-t)W_1(\mu_g,\mu_D)$, where $\psi_t = (1-t)id+t\psi$.
\end{proposition}

The notion of the optimal transport is briefly introduced in Appendix.
By this proposition, there exists a curve $\psi_t$ strictly reducing distance, i.e., $dW_1(\psi_{t\sharp}\mu_g, \mu_D)/dt < 0$ if $\mu_g \centernot= \mu_D$.
Because $\mathcal{L}$ approximates $W_1$, it is expected that $d \mathcal{L}(\psi_t)/dt < 0$ when $\mu_g$ differs from $\mu_D$.
Noting that $d \mathcal{L}(\psi_t)/dt = \pd<\nabla_\phi \mathcal{L}(\psi_t),\psi - id>_{L^2(\mu_g)}$, the functional gradient $\nabla_\phi \mathcal{L}(\psi_t) \centernot = 0$
does not vanish and the objective $\mathcal{L}$ may be strictly reduced by Algorithm \ref{FGD}.

\section{Gradient Flow Perspective}
In Euclidean space, the step of the steepest descent method for minimizing problems can be derived by the discretization of the gradient flow $d\gamma(t)/dt=-\nabla_xF(\gamma(t))$ where $F$ is an objective function on Euclidean space.
Because our goal is to move $\mu_g$ closer to $\mu_D$, we should consider a gradient flow in the space of probability measures.
To make this argument rigorously, we need the continuity equation that characterizes a curve of probability measures and the tangent space where velocities of curves should be contained (c.f., \cite{ambrosio2008gradient}).
When these notions are provided, the gradient flow is defined immediately and it is quite natural to discretize this flow to track it well.
In this section, we show that Algorithm \ref{FGD} is such a natural discretization; in other words, building a deep neural network by stacking gradient layers is a discretization procedure of the gradient flow.
We refer to \cite{ambrosio2008gradient} for detailed descriptions on this subject, and also refer to \cite{otto2001geometry} for an original method developed by Otto.

\subsection{Continuity Equation and Discretization}
We denote by $\mathcal{P}$ the set of probability measures on $\mathbb{R}^v$.
For $\mu \in \mathcal{P}$, let $\{\phi_t\}_{t\in [0,\delta]}$ be a curve in $L^2(\mu)$ that solves the following ordinary differential equation: for an $L^2(\phi_{t\sharp}\mu)$-integrable
vector field $v_t$ on $\mathbb{R}^v$, 
\begin{equation*}
  \phi_0 = id,\ \ \frac{d}{dt}\phi_t(x) = v_t( \phi_t(x))\ for \ \forall x \in \mathbb{R}^v.
\end{equation*}
Then, this equation derives the curve $\nu_t=\phi_{t\sharp}\mu$ in $\mathcal{P}$, which can be characterized by .
\begin{equation}
  \frac{d}{dt}\nu_t + \nabla \cdot (v_t\nu_t) = 0. \label{continuity_equation}
\end{equation}
In other words, the following equation is satisfied
\begin{equation*}
  \int_{I} \int_{\mathbb{R}^v} (\partial_t f(x,t) + \nabla_{x}f(x,t)^\top v_t) d\nu_t dt = 0,
\end{equation*}
for $\forall f \in \mathcal{C}^{\infty}_c(\mathbb{R}^v \times I )$ where $\mathcal{C}_{c}^\infty(\mathbb{R}^v\times I)$ is the set of $\mathcal{C}^\infty$-functions with compact support in $\mathbb{R}^v\times I$.
Conversely, a {\it narrowly} continuous family of probability measures $\nu_t$ solving equation (\ref{continuity_equation}) can be obtained
by transport map $\phi_t$ satisfying $\frac{d}{dt}\phi_t(x) = v_t( \phi_t(x))$ \cite{ambrosio2008gradient}.
Thus, equation (\ref{continuity_equation}) indicates that $v_t$ drifts the probability measures $\nu_t$. 
Indeed, $v_t$ can be recognized as the tangent vector of the curve $\nu_t$ as discussed below.

Here, we focus on curves in the subset $\mathcal{P}_2 \subset \mathcal{P}$ composed of probability measures with finite second moment.
Noting that there is freedom in the choice of $v_t$ modulo divergence-free vector fields $w \in L^2(\nu_t)$ (i.e., $\nabla\cdot(w\nu_t)=0$),
it is natural to consider the equivalence class of $v \in L^2(\nu_t)$ modulo divergence-free vector fields.
Moreover, there exists a unique $\Pi(v)$ that attains the minimum $L^2(\nu_t)$-norm in this class:
$\Pi(v)=\arg\min_{w \in L^2(\nu_t)} \{ \|w\|_{L^2(\nu_t)} \mid \nabla\cdot((v-w)\nu_t)=0\}$.
Thus, we here introduce the definitions of the tangent space at $\mu \in \mathcal{P}_2$ as follows:
\begin{equation}
  T_\mu \mathcal{P}_2 \overset{\mathrm{def}}{=} \{ \Pi(v) \mid v \in L^2(\mu)\}. \label{tangent_space}
\end{equation}

The following proposition shows that $T_\mu \mathcal{P}_2$ has the property of the tangent space on the space of probability measures, that is,
a perturbation using $v_t \in T_\mu \mathcal{P}_2$ can discretize an absolutely continuous curve $\nu_t$ and $v_t$ locally approximates {\it optimal transport maps}.
We denote the $2$-Wasserstein distance by $W_2$.

\begin{proposition}[\cite{ambrosio2008gradient}] \label{ac_prop}
  Let $\nu_t: I \rightarrow \mathcal{P}_2$ be an absolutely continuous curve satisfying the continuity equation with a Borel vector field $v_t$ that is contained in $T_{\nu_t}\mathcal{P}_2$ almost everywhere $t \in I$.
  Then, for almost everywhere $t \in I$ the following property holds:
  \begin{equation*}
    \lim_{\delta \rightarrow 0} \frac{W_2(\nu_{t+\delta},(id+\delta v_t)_\sharp \nu_t)}{|\delta|} = 0.
  \end{equation*}
  In particular, for almost everywhere $t\in I$ such that $\nu_t$ is absolutely continuous with respect to the Lebesgue measure, we have
  \begin{equation*}
    \lim_{\delta \rightarrow 0} \frac{1}{\delta} (\mathbf{t}_{\nu_t}^{\nu_{t+\delta}}- id) = v_t \quad in\ L^2(\nu_t),
  \end{equation*}
  where $\mathbf{t}_{\nu_t}^{\nu_{t+\delta}}$ is the unique optimal transport map between $\nu_t$ and $\nu_{t+\delta}$.
\end{proposition}

This proposition suggests the update $\mu^+ \leftarrow (id+v)_\sharp \mu$ for discretizing an absolutely continuous curve in $\mathcal{P}_2$.
Note that when $\mu=\phi_\sharp \nu$, ($\nu \in \mathcal{P}_2, \phi \in L^2(\nu)$), the corresponding map to $\mu^+$ is obtained by $\phi^+_\sharp \nu = \mu^+$ where $\phi^+$ is a composition as follows:
\begin{equation}
  \phi^+ \leftarrow (id+v)\circ \phi = \phi + v\circ \phi. \label{update_of_transport_map}
\end{equation}

So far, we have introduced the property of continuous curves in $\mathcal{P}_2$ and a method of their discretization.
We notice that the above update resembles the update of Algorithm \ref{FGD}.
Indeed, we show that the functional gradient method is nothing but a discretization method of the gradient flow derived by the functional gradient $\nabla_\phi \mathcal{L}(\phi)$.

\subsection{Discretization of Gradient Flow}
We here introduce the gradient flow, which is one of the most straightforward ways to understand Algorithm \ref{FGD}.
We have explained that an absolutely continuous curve $\{\nu_t\}_{t\in I}$ in $\mathcal{P}_2$ is well characterized by the continuity equation (\ref{continuity_equation}) and we have seen that
$\{v_t\}_{t \in I}$ in (\ref{continuity_equation}) corresponds to the notion of the velocity field induced by the curve.
Such a velocity points in the direction of the particle flow.
Moreover, the functional gradient $\nabla_\phi \mathcal{L}(\phi)(\cdot)$ points in an opposite direction to reduce the objective $\mathcal{L}$ at each particle.
Thus, these two vector fields exist in the same space and it is natural to consider the following equation:
\begin{equation}
  v_t = -\nabla_\phi \mathcal{L}(\phi_t). \label{grad_flow}
\end{equation}
This equation for an absolutely continuous curve is called the gradient flow \cite{ambrosio2008gradient} and a curve satisfying this will reduce the objective $\mathcal{L}$.
Indeed, we can find by the chain rule such a curve $\{\nu_t=\phi_{t\sharp}\mu_g\}_{t \in I}$ that also satisfies the following:
\begin{equation*}
  \frac{d}{dt} \mathcal{L}(\phi_t) = - \| \nabla_\phi \mathcal{L}(\phi_t) \|_{L^2(\nu_t)}^2.
\end{equation*}

Recalling that $\nu_t$ can be discretized well by $\nu_{t+\delta} \sim (id-\delta \nabla_\phi \mathcal{L}(\phi_t))_\sharp \nu_t$,
we notice that Algorithm \ref{FGD} is a discretization method of the gradient flow (\ref{grad_flow}).
In other words, building deep neural networks by stacking gradient layers is such a discretization procedure.

\begin{figure*}[th]
  \begin{center} 
    \begin{tabular}{cccc|c}
 \includegraphics[width=24mm,angle=0]{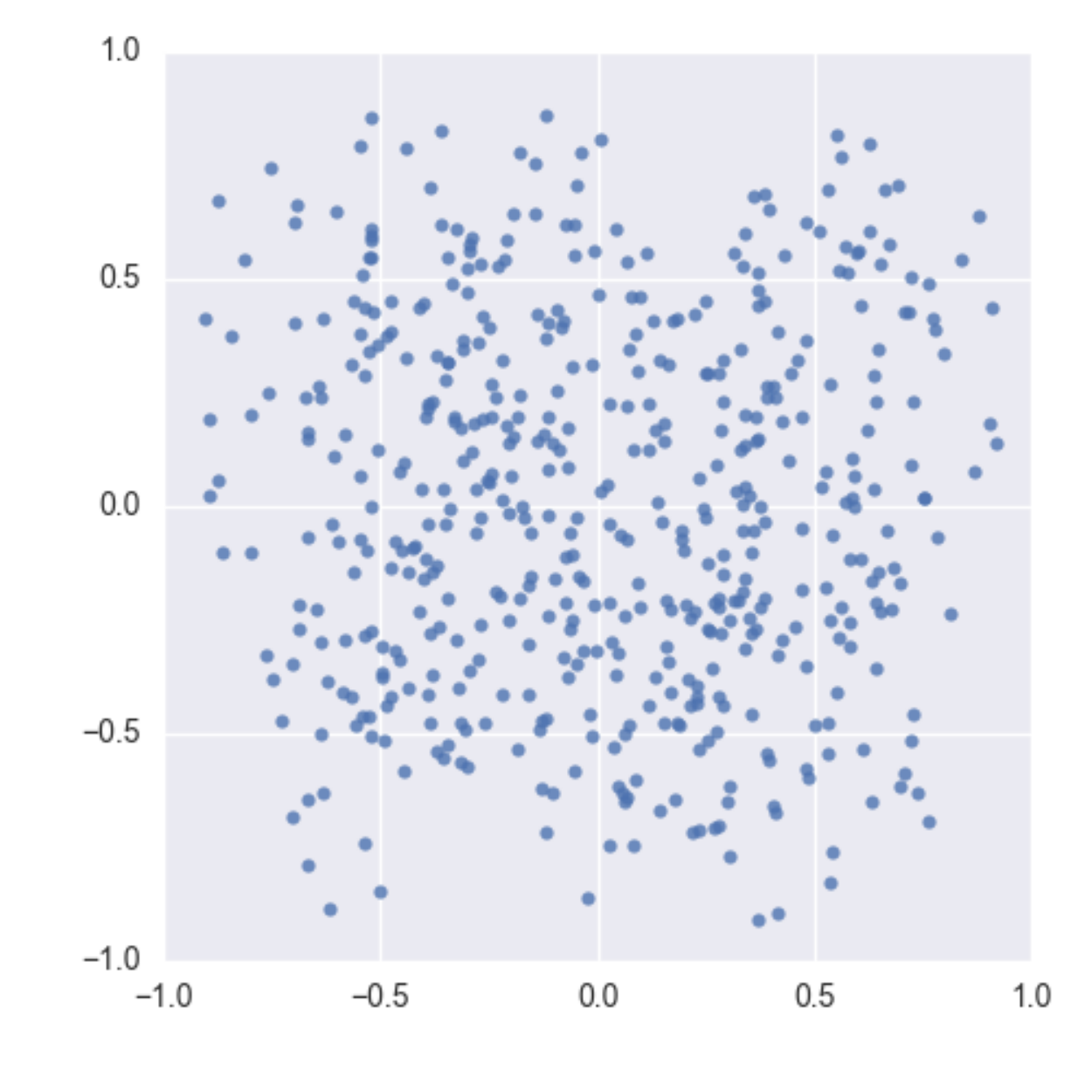} 
 \includegraphics[width=24mm,angle=0]{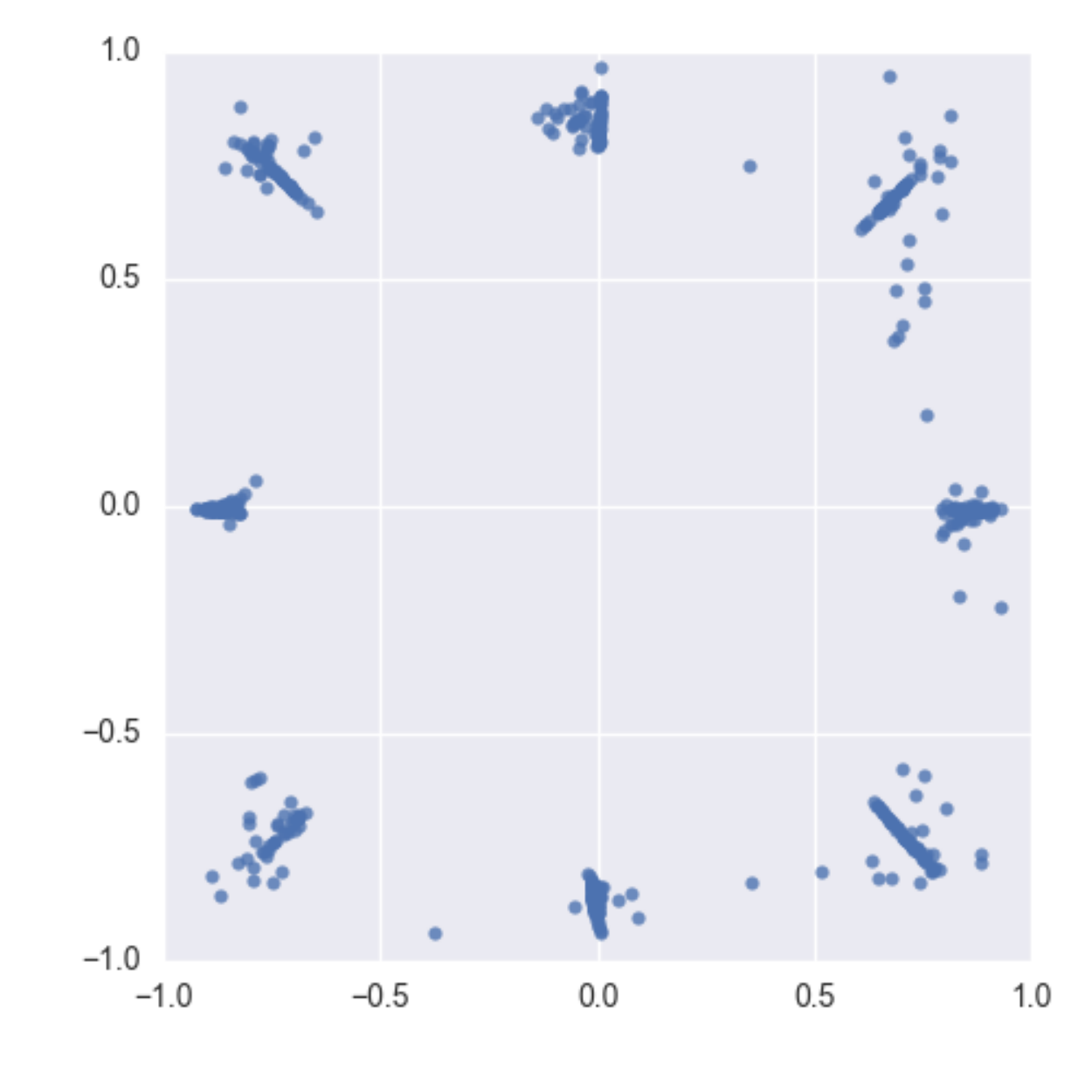} 
 \includegraphics[width=24mm,angle=0]{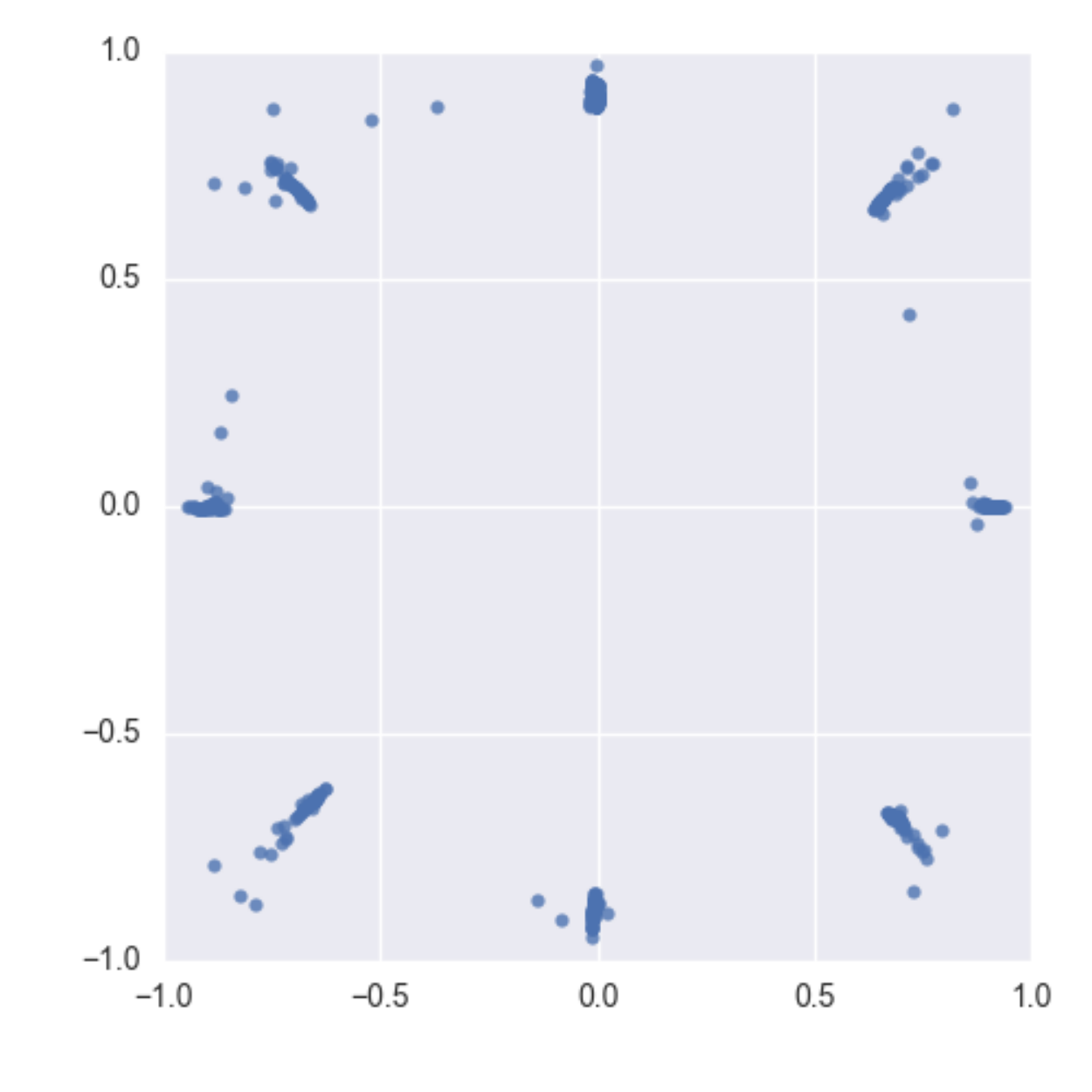} 
 \includegraphics[width=24mm,angle=0]{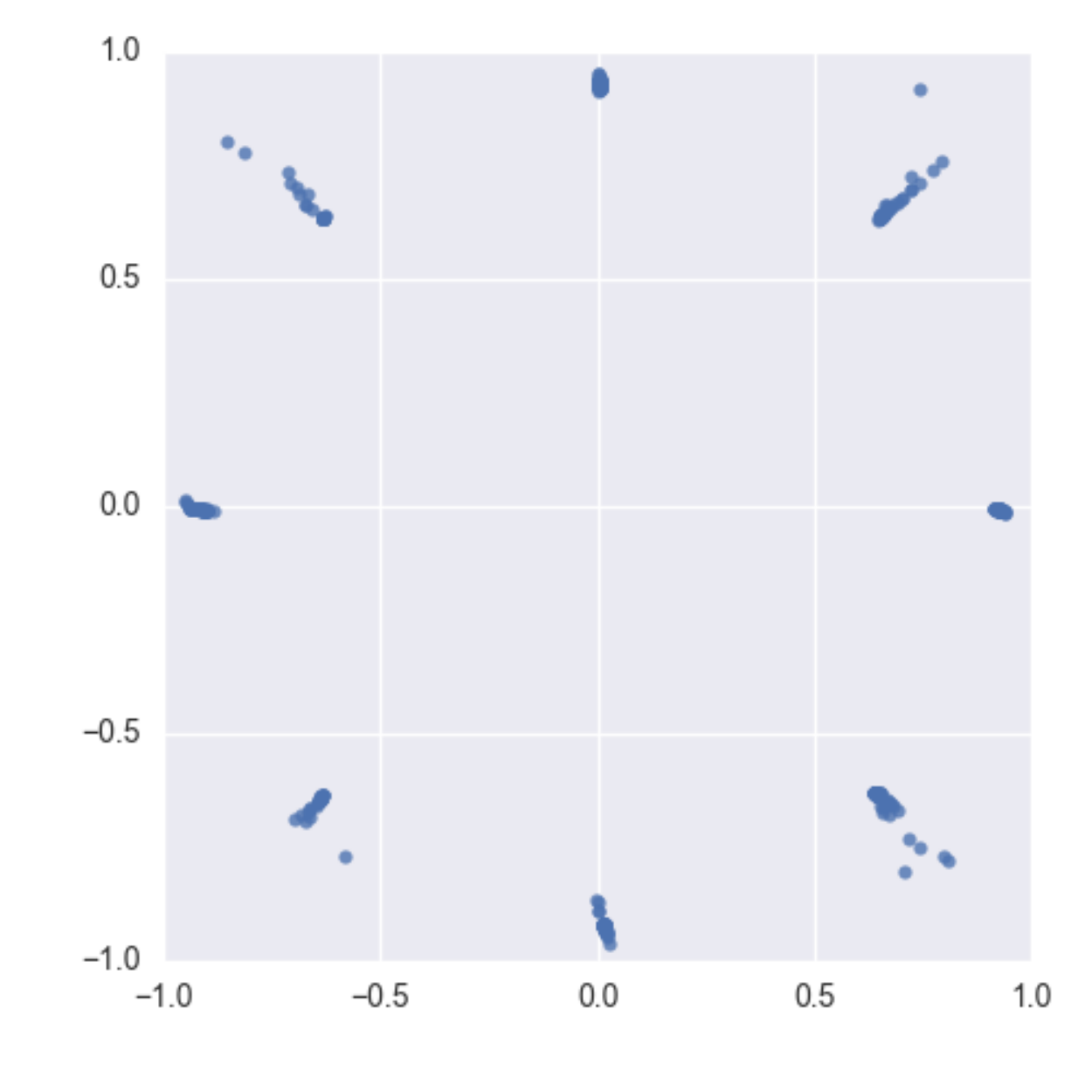} 
 \includegraphics[width=24mm,angle=0]{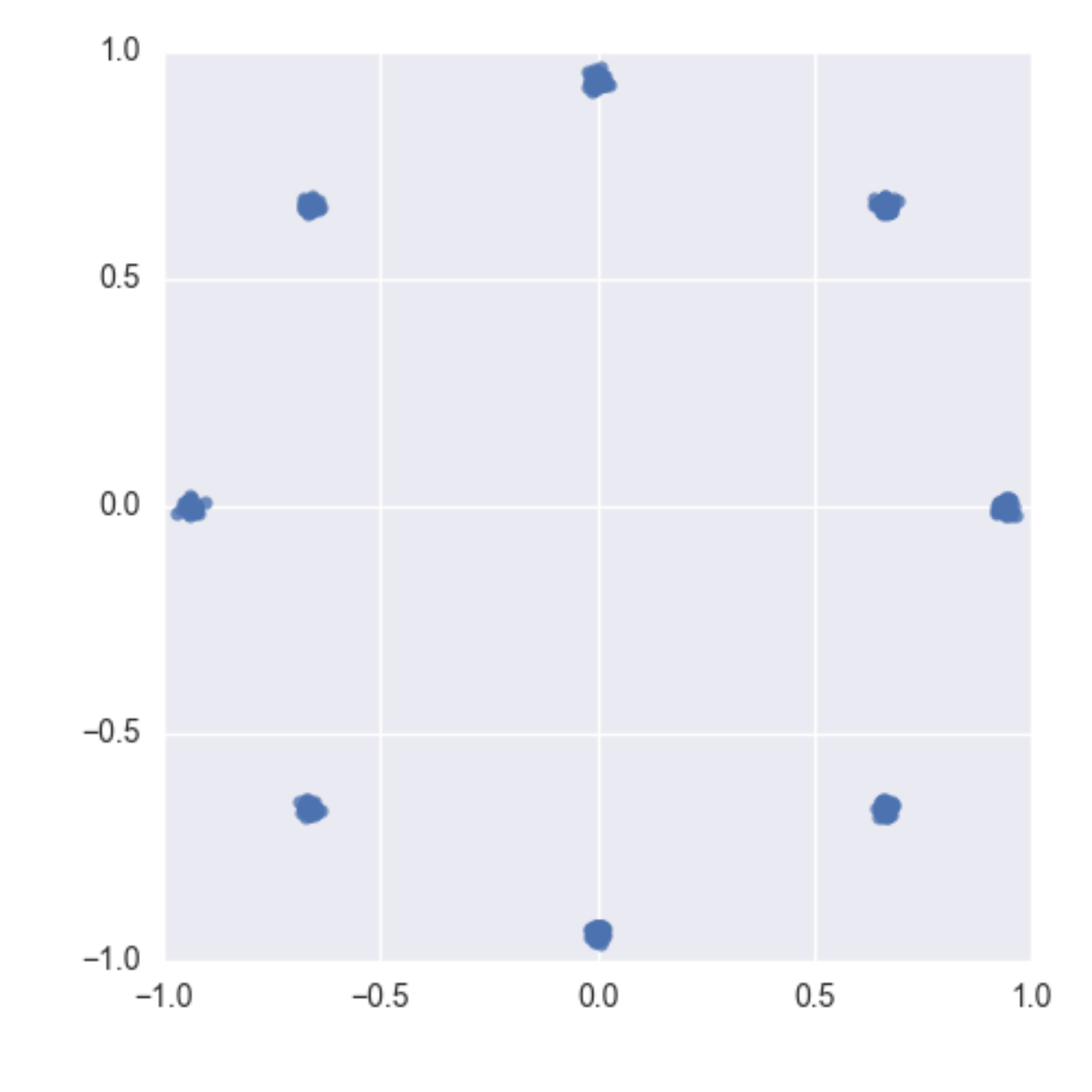}  
 \vspace{-4mm}\\
 \includegraphics[width=24mm,angle=0]{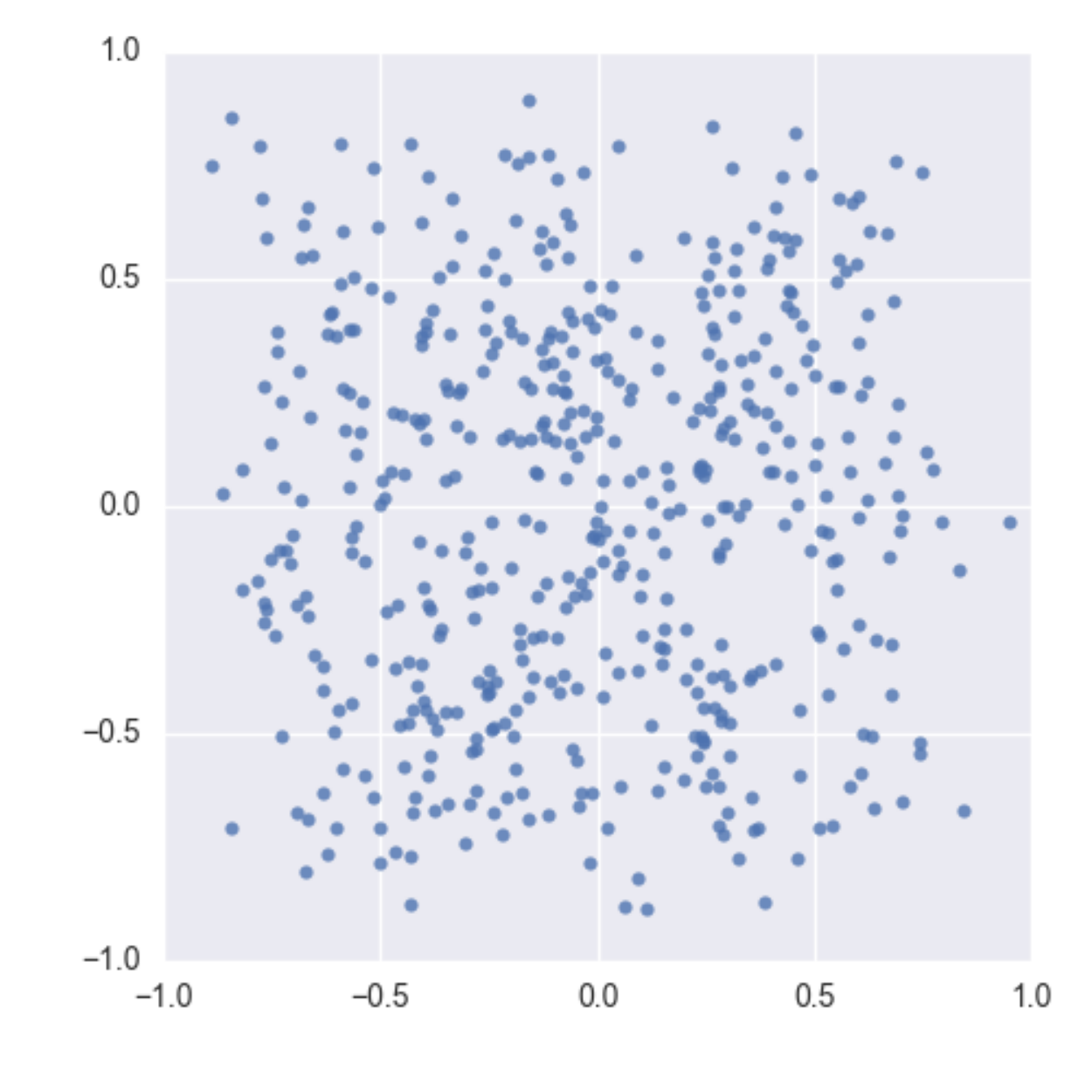} 
 \includegraphics[width=24mm,angle=0]{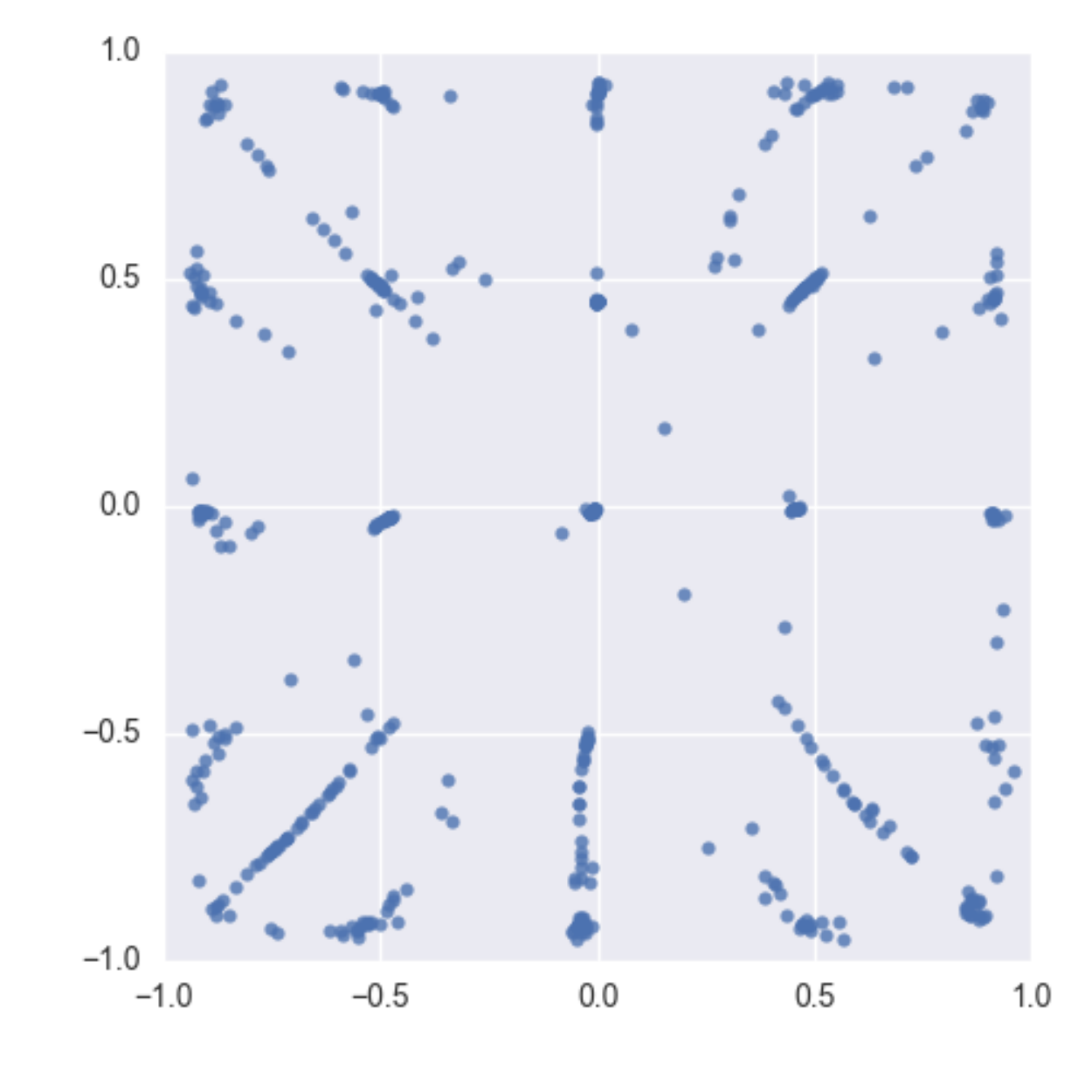} 
 \includegraphics[width=24mm,angle=0]{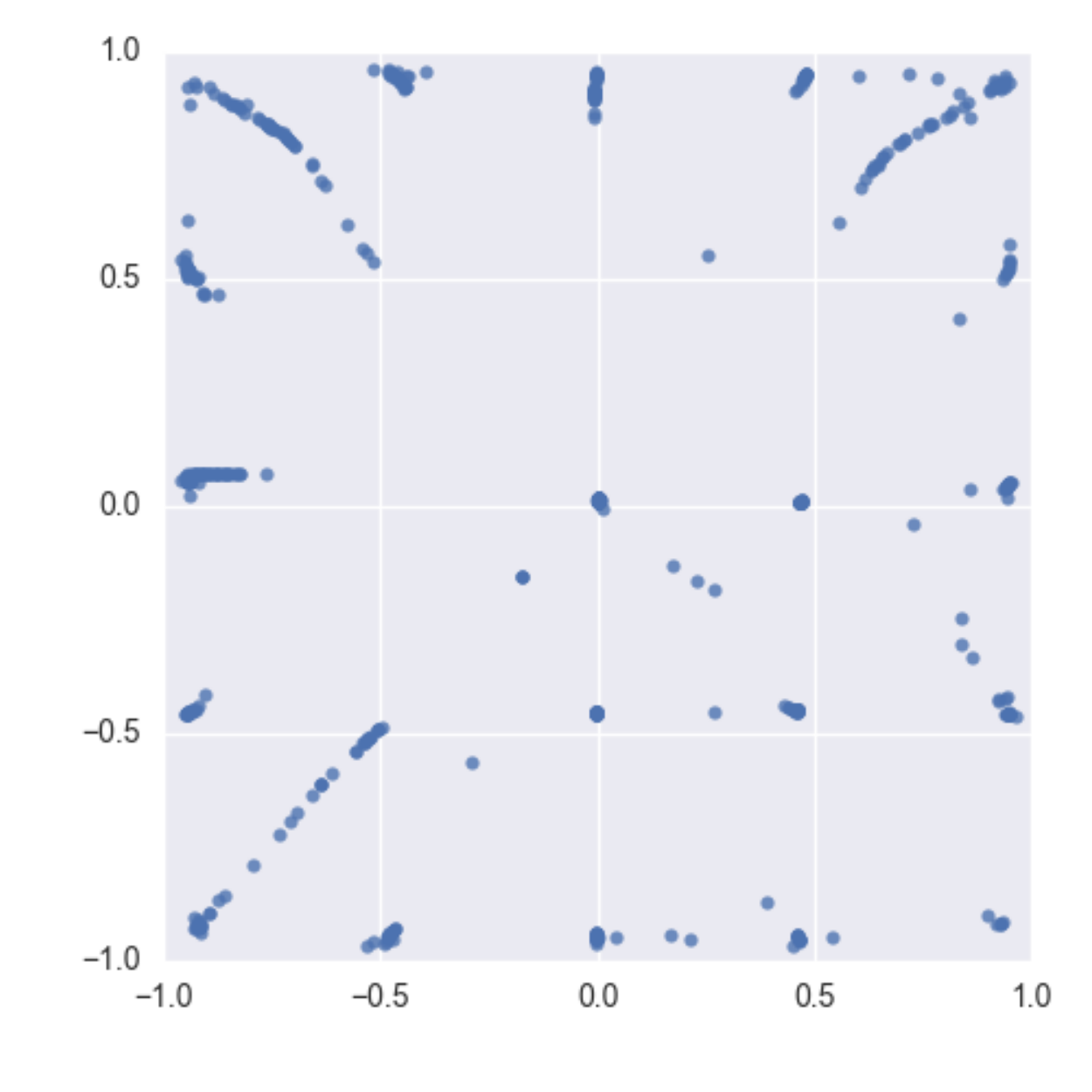} 
 \includegraphics[width=24mm,angle=0]{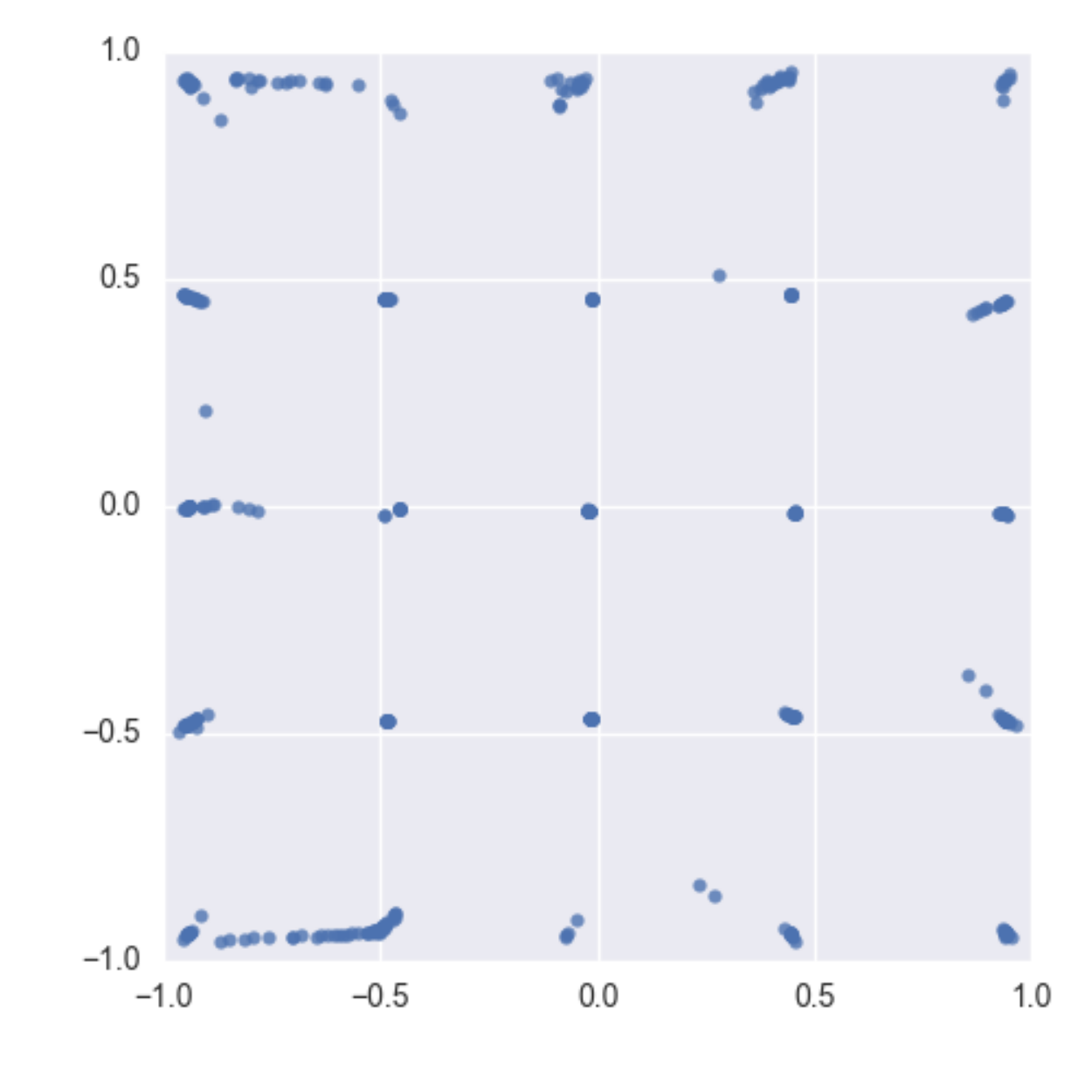} 
 \includegraphics[width=24mm,angle=0]{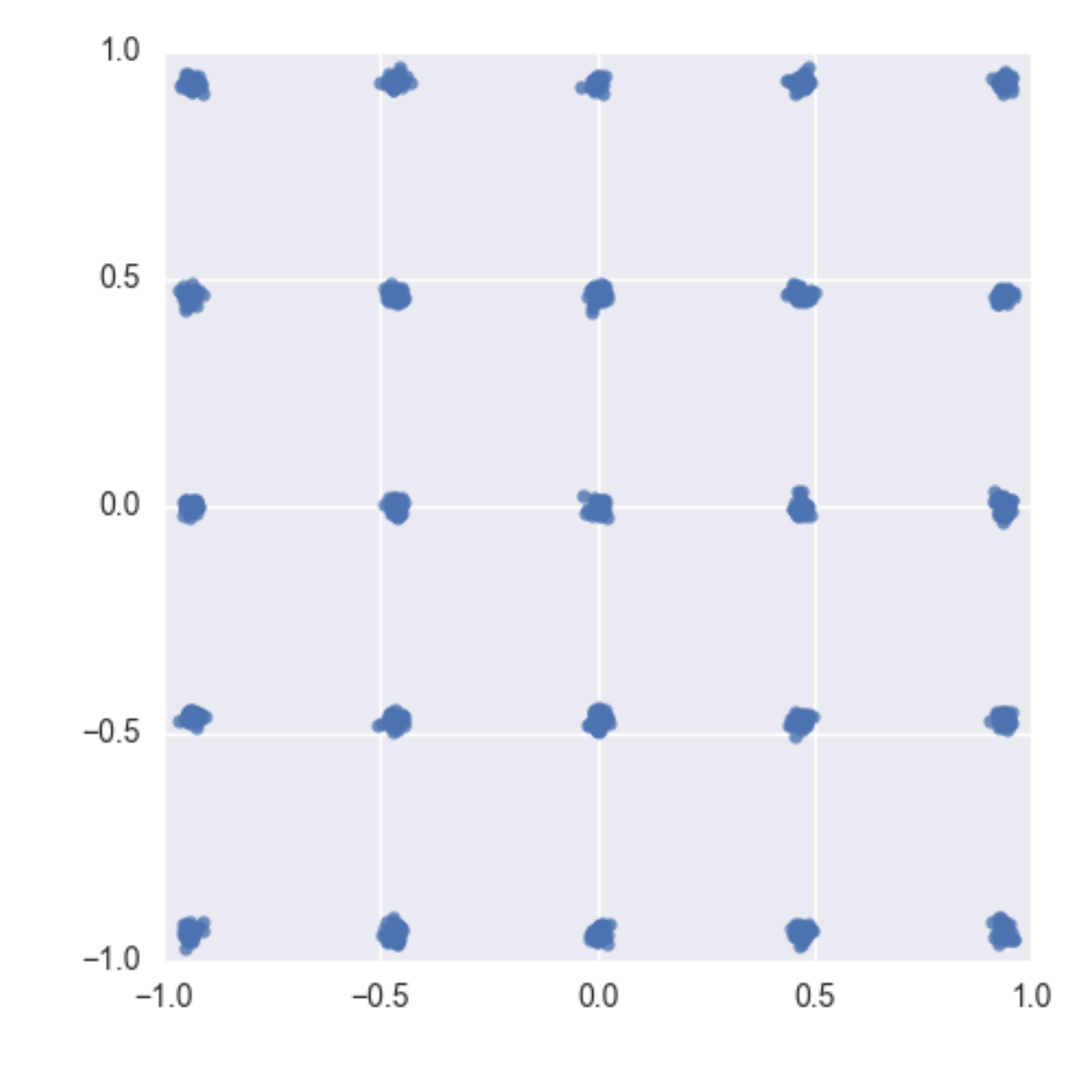}
 \vspace{-4mm}\\
 \includegraphics[width=24mm,angle=0]{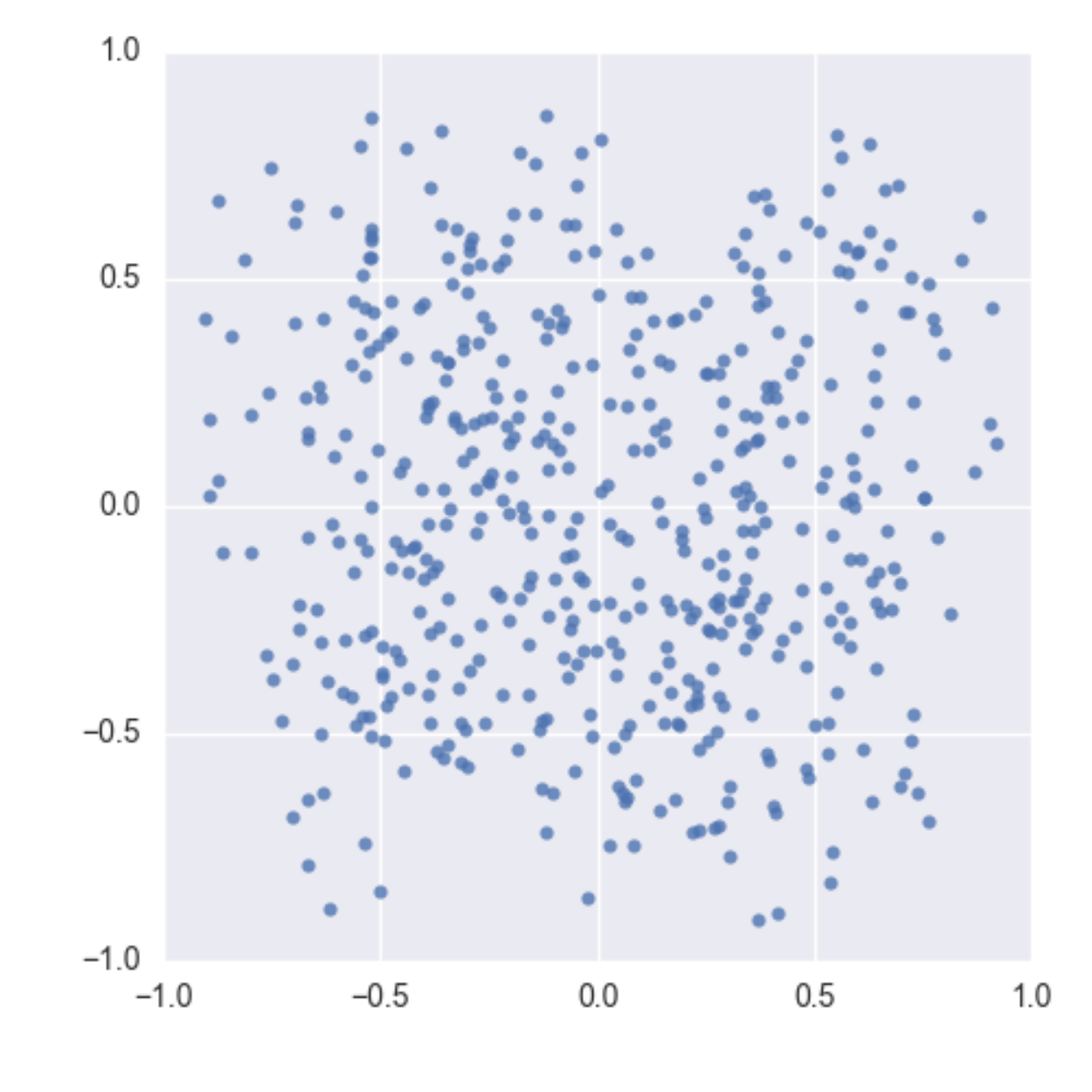} 
 \includegraphics[width=24mm,angle=0]{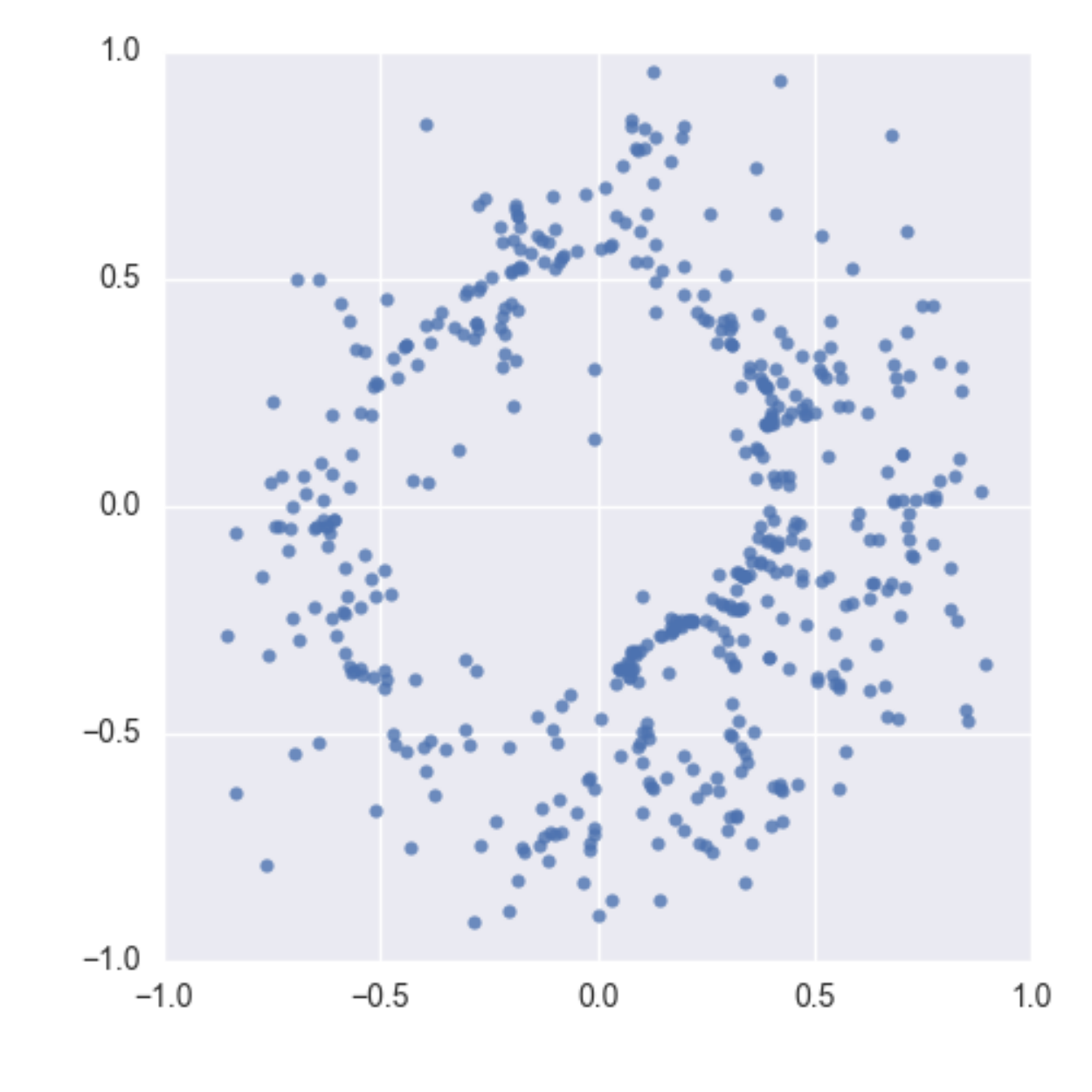} 
 \includegraphics[width=24mm,angle=0]{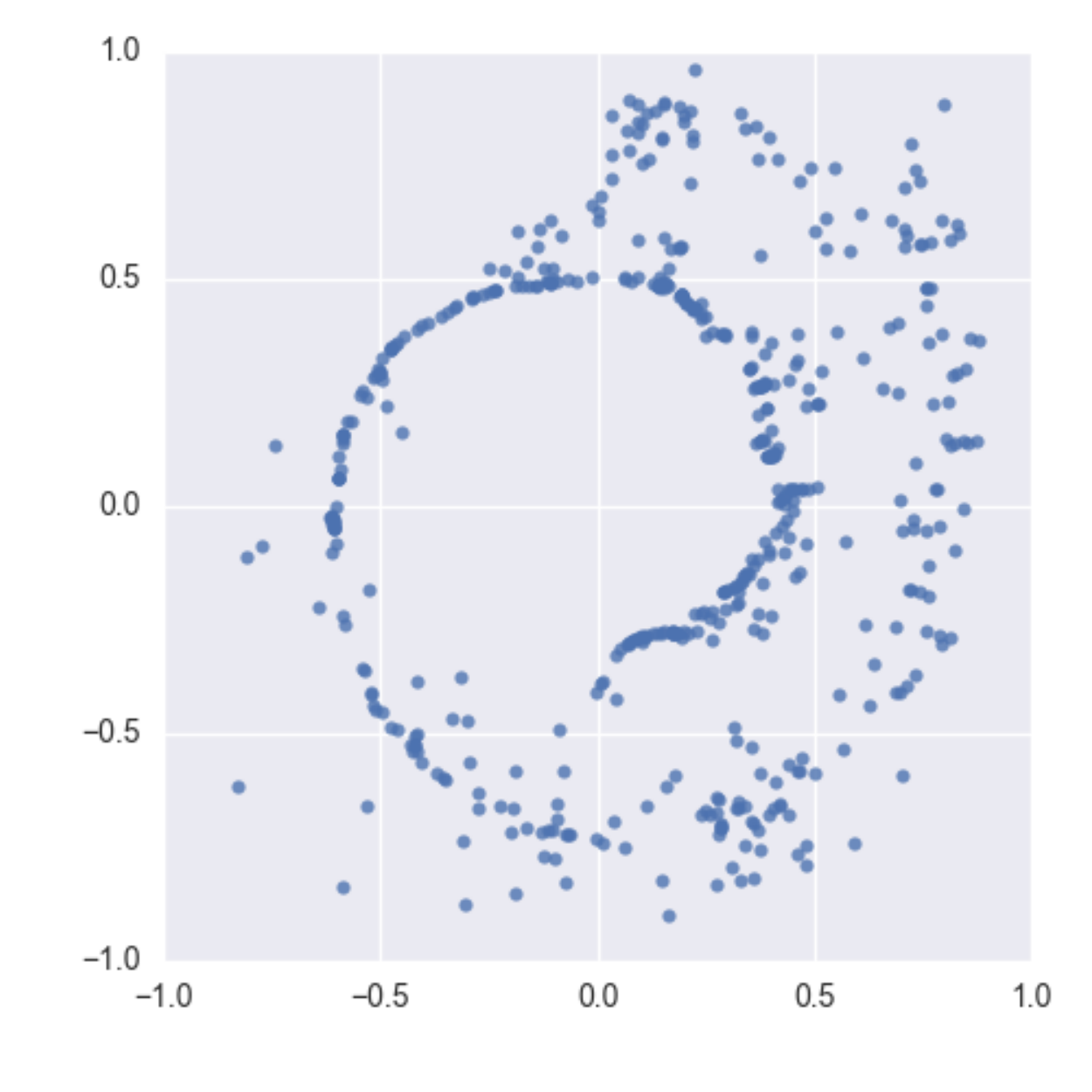} 
 \includegraphics[width=24mm,angle=0]{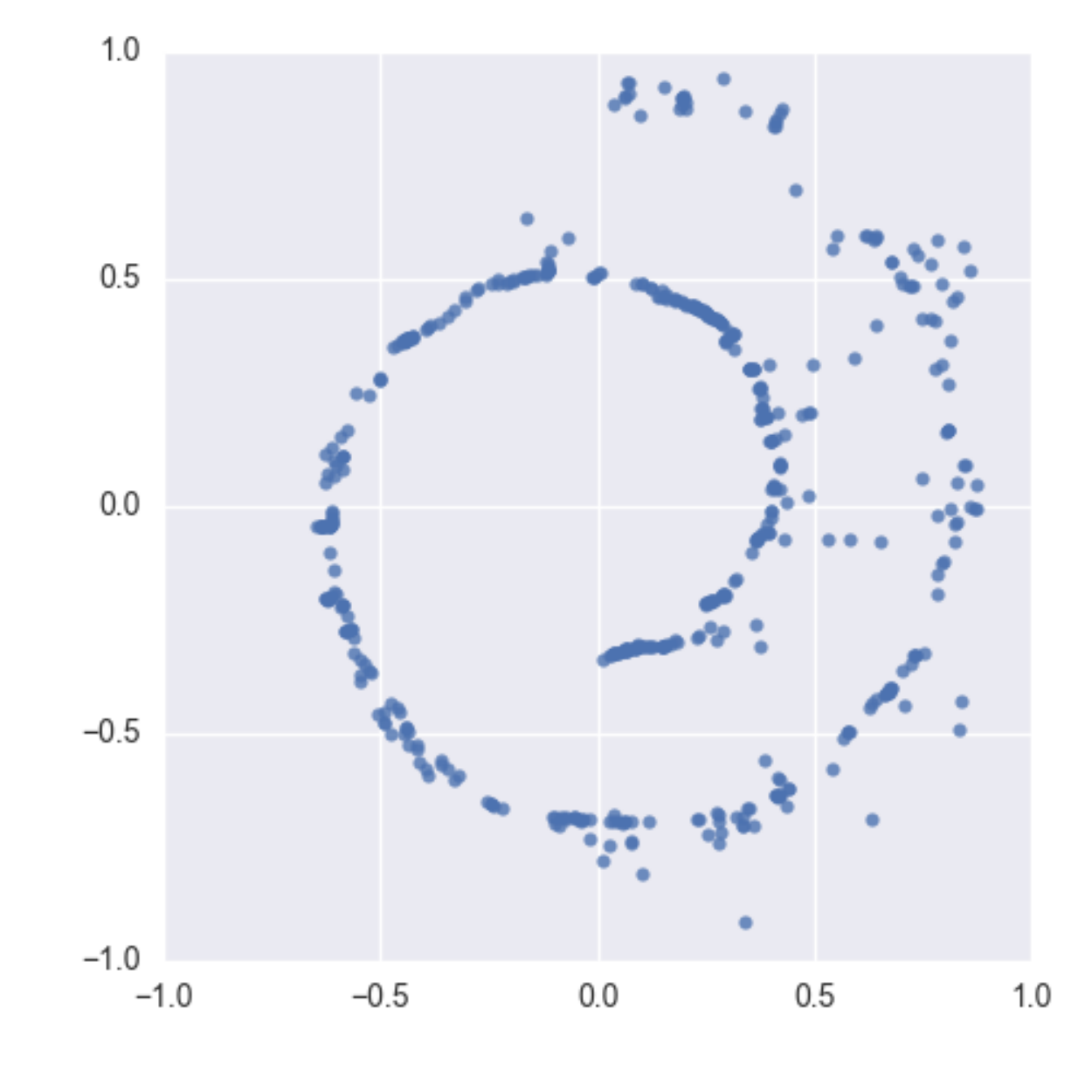} 
 \includegraphics[width=24mm,angle=0]{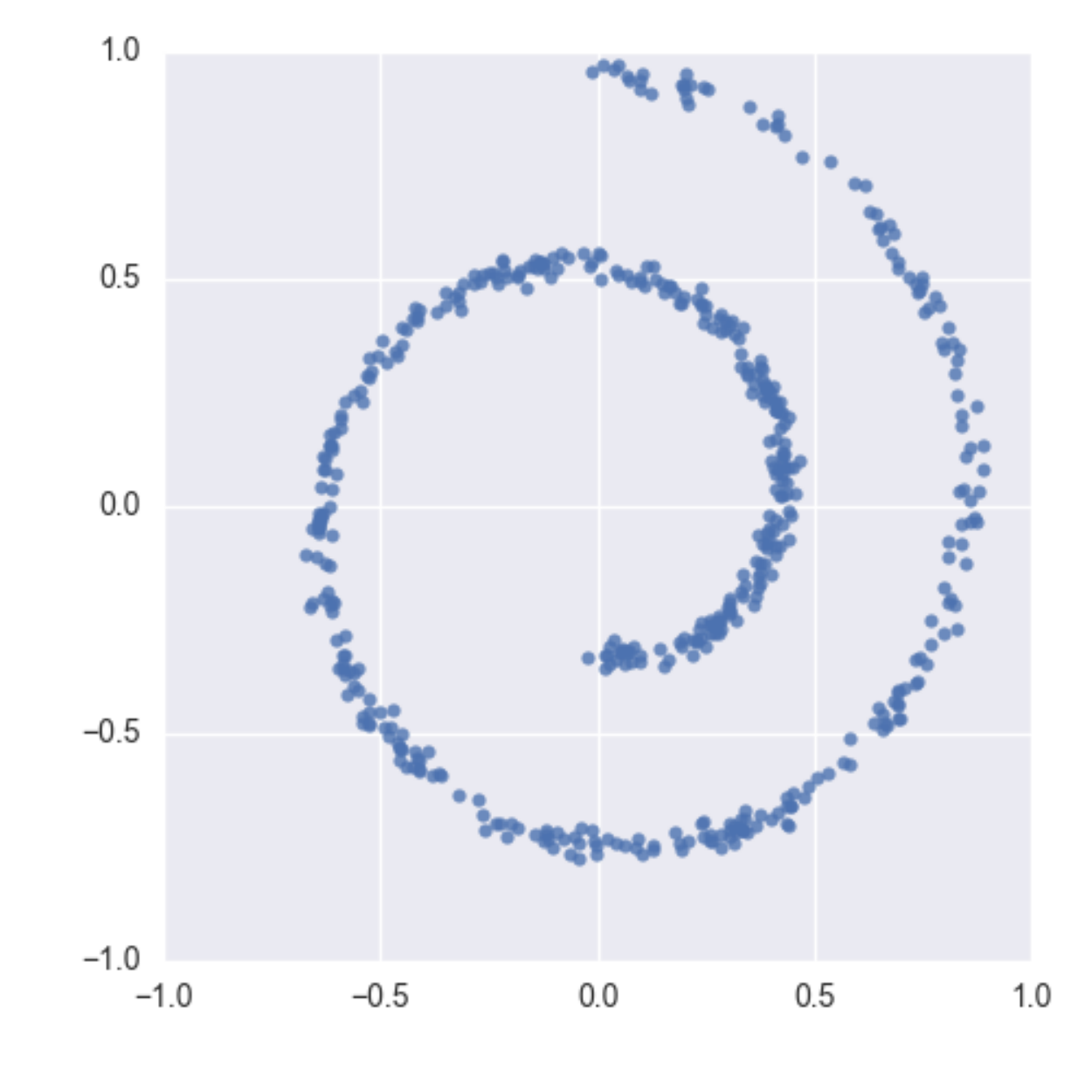} \\            
    \end{tabular}
    \caption{Generated samples by Algorithm \ref{WGAN_FT} on $8$-gaussian dataset for $0, 25, 50$, and $100$ generator iterations (from left to right) and training data (rightmost).} \label{toy_datasets}
\end{center}
\end{figure*}

\section{Experiments} \label{sec:experiments_sec}
In this section, we show the powerful optimization ability of the gradient layer method empirically on training WGANs.
Our implementation is done using Theano \cite{2016arXiv160502688short}.
We first used three toy datasets: swiss roll, $8$-gaussian, and $25$-gaussian datasets (see Figure \ref{toy_datasets}) to confirm the convergence behavior of the gradient layer.
The sizes of toy datasets are $500$, $500$, and $1000$, respectively.
We next used the CIFAR-10 containing 50,000 images of size 32$\times$32,
and STL-10 containing 100,000 images.
For STL-10 dataset, we downsample each dimension by 2, resulting image size is 48$\times$48.
We reported inception scores \cite{salimans2016improved} for image datasets, which is one of the conventional scores commonly used to measure the quality of generated samples. 
\paragraph{Toy datasets}
We ran Algorithm \ref{WGAN_FT} without pre-training of generators (i.e., $g=id$) on toy datasets from Gaussian noise distributions with the standard deviation $0.5$.
We used four-layer neural networks for the critics where the dimension of hidden layers were set to $128$ for swiss roll and $8$-gaussian datasets and $512$ for $25$-gaussian dataset.
We adopted one-sided penalty with regularization parameter $\lambda=10$.
The output of generator was activated by ${\rm tanh}$.
We used ADAM for training critics with parameters $\alpha=10^{-4}, \beta_1=0.5, \beta_2=0.9$, minibatch size $b=50$.
When we run Algorithm \ref{WGAN_FT}, gradient layers are stacked below $\tanh$.
The learning rates were set to $\eta = 0.1$.
The number of inner iterations $T_0$ for training the critics was set to $5 \times datasize / b$.
Figure \ref{toy_datasets} shows the results for toy datasets for running $T=100$ iterations of generators.
Although we ran the algorithm without pre-training the generators, we obtained better results only for a few iterations.
This is surprising, because these toy datasets are difficult to learn and fail to converge in the standard GANs and WGANs.
Whereas improved variants of these models overcome this difficulty, they usually require more than 1,000 iterations to converge.

\begin{figure}[t]
  \begin{center}
    \includegraphics[width=75mm,angle=0]{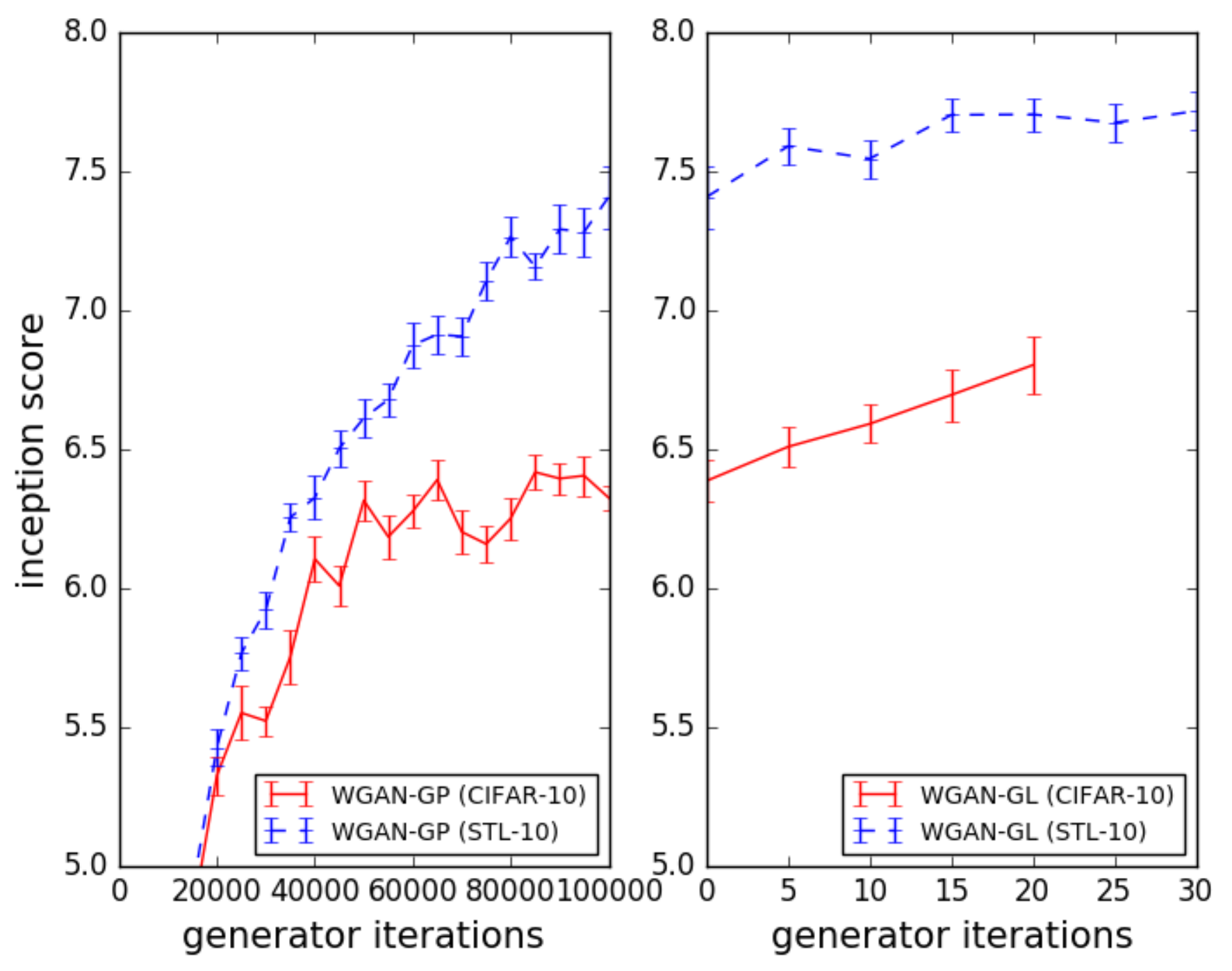}    
      \caption{Left: Inception scores obtained by WGAN-GP, Right: Inception scores obtain by Algorithm \ref{WGAN_FT}
        starting from the result of WGAN-GP.} \label{inception_result}
  \end{center}    
\end{figure}
\paragraph{CIFAR-10 and STL-10}
We first trained WGAN-GP with a two-sided penalty ($\lambda=10$) on the CIFAR-10 and STL-10 datasets.
We used DCGAN for both the critic and the generator.
The batch normalization \cite{ioffe2015batch} was used only for the generator.
The critic and the generator were trained by using ADAM with $\alpha=10^{-4}, \beta_1=0.5, \beta_2=0.9$, and minibatch size $b=64$.
The number of inner iterations for training the critics were $5$ and we ran ADAM for $10^5$-iterations.
The left side of Figure \ref{inception_result} shows the inception scores obtained by WGAN-GP.
It seems that the learning procedure is slowed down in a final training phase, especially for CIFAR-10.
The final inception score on CIFAR-10 and STL-10 are $6.32$ and $7.40$, respectively.
We next ran Algorithm \ref{WGAN_FT} starting from the result of WGAN-GP.
The critics were trained by ADAM with the same parameters, except for $\alpha=5\times10^{-5}$ and $T_0=datasize / b$.
The learning rates were set to $0.5$ for CIFAR-10 and $0.3$ for STL-10.
The right side of Figure \ref{inception_result} shows the inception scores obtained by Algorithm \ref{WGAN_FT}.
Note that, since we focus on the optimization ability of generators, we plotted results with the horizontal axis as the number of outer-iterations. 
We observed a rapid increase in the inception scores, which were improved to $6.80$ and $7.71$ on CIFAR-10 and STL-10, respectively.

\section{Conclusion}
We have proposed a gradient layer that enhances the convergence speed of adversarial training.
Because this layer is based on the perspective of infinite-dimensional optimization, it can avoid local optima induced by non-convexity and parameterization.
We have also provided two perspectives of the gradient layer: (i) the functional gradient method and 
(ii) the discretization procedure of the gradient flow.
We have proven the fast convergence of the gradient layer by utilizing this perspective, and experimental results have empirically shown its reliable performance.

\section*{\ackname}
This work was partially supported by MEXT KAKENHI (25730013, 25120012, 26280009, 15H01678 and 15H05707), JST-PRESTO (JPMJPR14E4), and JST-CREST (JPMJCR14D7, JPMJCR1304).

\bibliographystyle{plain}


\begin{thebibliography}{10}

\bibitem{ambrosio2003lecture}
Luigi Ambrosio.
\newblock Lecture notes on optimal transport problems.
\newblock In {\em Mathematical aspects of evolving interfaces}, pages 1--52.
  Springer, 2003.

\bibitem{ambrosio2008gradient}
Luigi Ambrosio, Nicola Gigli, and Giuseppe Savar{\'e}.
\newblock {\em Gradient Flows in Metric Spaces and in the Space of Probability
  Measures}.
\newblock Lectures in Mathematics. ETH Z{\"u}rich. Birkh{\"a}user Basel, 2008.

\bibitem{arjovsky2017wasserstein}
Martin Arjovsky, Soumith Chintala, and L{\'e}on Bottou.
\newblock Wasserstein generative adversarial networks.
\newblock In {\em International Conference on Machine Learning}, pages
  214--223, 2017.

\bibitem{brenier1987decomposition}
Yann Brenier.
\newblock D{\'e}composition polaire et r{\'e}arrangement monotone des champs de
  vecteurs.
\newblock {\em CR Acad. Sci. Paris S{\'e}r. I Math}, 305(19):805--808, 1987.

\bibitem{brenier1991polar}
Yann Brenier.
\newblock Polar factorization and monotone rearrangement of vector-valued
  functions.
\newblock {\em Communications on pure and applied mathematics}, 44(4):375--417,
  1991.

\bibitem{NIPS2016_6399}
Xi~Chen, Xi~Chen, Yan Duan, Rein Houthooft, John Schulman, Ilya Sutskever, and
  Pieter Abbeel.
\newblock Info{GAN}: Interpretable representation learning by information
  maximizing generative adversarial nets.
\newblock In {\em Advances in Neural Information Processing Systems 29}, pages
  2172--2180. 2016.

\bibitem{GAN2014}
Ian Goodfellow, Jean Pouget-Abadie, Mehdi Mirza, Bing Xu, David Warde-Farley,
  Sherjil Ozair, Aaron Courville, and Yoshua Bengio.
\newblock Generative adversarial nets.
\newblock In {\em Advances in Neural Information Processing Systems 27}. 2014.

\bibitem{gulrajani2017improved}
Ishaan Gulrajani, Faruk Ahmed, Martin Arjovsky, Vincent Dumoulin, and Aaron~C
  Courville.
\newblock Improved training of wasserstein gans.
\newblock In {\em Advances in Neural Information Processing Systems 30}, pages
  5769--5779, 2017.

\bibitem{he2016deep}
Kaiming He, Xiangyu Zhang, Shaoqing Ren, and Jian Sun.
\newblock Deep residual learning for image recognition.
\newblock In {\em Proceedings of the IEEE conference on computer vision and
  pattern recognition}, pages 770--778, 2016.

\bibitem{ioffe2015batch}
Sergey Ioffe and Christian Szegedy.
\newblock Batch normalization: Accelerating deep network training by reducing
  internal covariate shift.
\newblock In {\em Proceedings of International Conference on Machine Learning
  32}, pages 448--456, 2015.

\bibitem{kingma2015adam}
Diederik~P. Kingma and Jimmy Ba.
\newblock Adam: A method for stochastic optimization.
\newblock In {\em Proceedings of the 3rd International Conference on Learning
  Representations (ICLR)}, 2015.

\bibitem{kodali2017train}
Naveen Kodali, Jacob Abernethy, James Hays, and Zsolt Kira.
\newblock How to train your {DRAGAN}.
\newblock {\em arXiv preprint arXiv:1705.07215}, 2017.

\bibitem{larsen2016autoencoding}
Anders Boesen~Lindbo Larsen, S{\o}ren~Kaae S{\o}nderby, Hugo Larochelle, and
  Ole Winther.
\newblock Autoencoding beyond pixels using a learned similarity metric.
\newblock In {\em Proceedings of international conference on Machine learning
  33}, pages 1558--1566, 2016.

\bibitem{luenberger1969optimization}
David~G Luenberger.
\newblock {\em Optimization by vector space methods}.
\newblock John Wiley \& Sons, 1969.

\bibitem{milgrom2002envelope}
Paul Milgrom and Ilya Segal.
\newblock Envelope theorems for arbitrary choice sets.
\newblock {\em Econometrica}, 70(2):583--601, 2002.

\bibitem{NIPS2016_6066}
Sebastian Nowozin, Botond Cseke, and Ryota Tomioka.
\newblock f-{GAN}: Training generative neural samplers using variational
  divergence minimization.
\newblock In {\em Advances in Neural Information Processing Systems 29}, pages
  271--279. 2016.

\bibitem{otto2001geometry}
Felix Otto.
\newblock The geometry of dissipative evolution equations: the porous medium
  equation.
\newblock 2001.

\bibitem{radford2016unsupervised}
Alec Radford, Luke Metz, and Soumith Chintala.
\newblock Unsupervised representation learning with deep convolutional
  generative adversarial networks.
\newblock In {\em Proceedings of International Conference on Learning
  Representations 4}, 2016.

\bibitem{salimans2016improved}
Tim Salimans, Ian Goodfellow, Wojciech Zaremba, Vicki Cheung, Alec Radford, and
  Xi~Chen.
\newblock Improved techniques for training {GAN}s.
\newblock In {\em Advances in Neural Information Processing Systems 29}, pages
  2234--2242, 2016.

\bibitem{sudakov1979geometric}
Vladimir~N Sudakov.
\newblock {\em Geometric problems in the theory of infinite-dimensional
  probability distributions}.
\newblock Number 141. American Mathematical Soc., 1979.

\bibitem{2016arXiv160502688short}
{Theano Development Team}.
\newblock {Theano: A {Python} framework for fast computation of mathematical
  expressions}.
\newblock {\em arXiv preprint arXiv:1605.02688}, 2016.

\bibitem{tieleman2012lecture}
Tijmen Tieleman and Geoffrey Hinton.
\newblock Lecture 6.5-rmsprop: Divide the gradient by a running average of its
  recent magnitude.
\newblock {\em COURSERA: Neural networks for machine learning}, 4(2):26--31,
  2012.

\bibitem{villani2008optimal}
C{\'e}dric Villani.
\newblock {\em Optimal transport: old and new}, volume 338.
\newblock Springer Science \& Business Media, 2008.

\bibitem{wang2016learning}
Dilin Wang and Qiang Liu.
\newblock Learning to draw samples: With application to amortized mle for
  generative adversarial learning.
\newblock {\em arXiv preprint arXiv:1611.01722}, 2016.

\bibitem{zhang2017stackgan}
Han Zhang, Tao Xu, Hongsheng Li, Shaoting Zhang, Xiaogang Wang, Xiaolei Huang,
  and Dimitris~N Metaxas.
\newblock Stack{GAN}: Text to photo-realistic image synthesis with stacked
  generative adversarial networks.
\newblock In {\em Proceedings of the IEEE Conference on Computer Vision and
  Pattern Recognition}, pages 5907--5915, 2017.

\end{thebibliography}

\clearpage
\onecolumn
\renewcommand{\thesection}{\Alph{section}}
\renewcommand{\thesubsection}{\Alph{section}. \arabic{subsection}}
\renewcommand{\thetheorem}{\Alph{theorem}}
\renewcommand{\thelemma}{\Alph{lemma}}
\renewcommand{\theproposition}{\Alph{proposition}}
\renewcommand{\thedefinition}{\Alph{definition}}
\renewcommand{\thecorollary}{\Alph{corollary}}
\renewcommand{\theassumption}{\Alph{assumption}}

\setcounter{section}{0}
\setcounter{subsection}{0}
\setcounter{theorem}{0}
\setcounter{lemma}{0}
\setcounter{proposition}{0}
\setcounter{definition}{0}
\setcounter{corollary}{0}
\setcounter{assumption}{0}

\part*{\Large{Appendix}}

\section{The Other Usage}
We introduce the usage that inserts a fixed number of gradient layers into the bottom of the generator to assist overall training procedure, 
which is described in Algorithm \ref{WGAN_GL}.
Note that we always use latest parameters of $f, g$ for gradient layers in Algorithm \ref{WGAN_GL}.
When gradient layers are inserted in the middle of the generator: $g_1\circ \phi \circ g_2$, we can apply Algorithm \ref{WGAN_GL} by setting $\mu_n \leftarrow g_{2\sharp}\mu_n, g \leftarrow g_1$.
After training, we can generate samples by using parameters of the critic and the generator, the learning rate, and the number of gradient layers,
which is described in Algorithm\ref{data_gen_proc_GL}.

\begin{algorithm}[h]
  \caption{Assisting WGAN-GP}
  \label{WGAN_GL}
\begin{algorithmic}
  \STATE {\bfseries Input:} The base distribution $\mu_n$, the minibatch size $b$, the number of iterations $T$,
  the initial parameters $\tau_0$ and $\theta_0$ of the critic and the generator,
  the number of iterations $T_0$ for the critic, the regularization parameter $c$, learning rate $\eta$ for gradient layers, the number of gradient layers $l$.\\
   \vspace{1mm}

   \FOR{$k=0$ {\bfseries to} $T-1$}
   \STATE $\tau \leftarrow \tau_k$
   \FOR{$k_0=0$ {\bfseries to} $T_0-1$}
   \STATE $\{x_i\}_{i=1}^b \sim \mu_D^b$, $\{z_i\}_{i=1}^b \sim \mu_n^b$, $\{\epsilon_i\}_{i=1}^b \sim U[0,1]^b$ \\
   \STATE \# $G_\eta^{\tau_k,\theta_k}$ is applied $l$ times.
   \STATE $\{z_i\}_{i=1}^b \leftarrow \{g_{\theta_k} \circ G^{\tau_k,\theta_k}_\eta\circ \cdots \circ G^{\tau_k,\theta_k}_\eta(z_i)\}_{i=1}^b$
   \STATE $\{\tilde{x}_i\}_{i=1}^b \leftarrow \{\epsilon_i x_i + (1-\epsilon_i)z_i\}_{i=1}^b$
   \STATE $v = \nabla_\tau \frac{1}{b}\sum_{i=1}^b [f_\tau(z_i)-f_\tau(x_i) + \lambda R_{f_\tau}(\tilde{x}_i)])$
   \STATE $\tau \leftarrow \mathcal{A}(\tau,v)$ \\
   \ENDFOR
   \STATE $\tau_{k+1} \leftarrow \tau$   
   \vspace{1mm}
   
   \STATE $\{z_i\}_{i=1}^b \sim \mu_n^b$
   \STATE \# $G_\eta^{\tau_{k+1},\theta_k}$ is applied $l$ times.
   \STATE $\{z_i\}_{i=1}^b \leftarrow \{ G_\eta^{\tau_{k+1},\theta_k} \circ \cdots \circ G_\eta^{\tau_{k+1},\theta_k}\}_{i=1}^b$
   \STATE $v \leftarrow - \nabla_\theta \frac{1}{b}\sum_{i=1}^b f_{\tau_{k+1}}( g_{\theta_k}(z_i) )$
   \STATE $\theta_{k+1} \leftarrow \mathcal{A}(\theta_k,v)$
   \ENDFOR
   \STATE Return $\tau_T, \theta_T$.
\end{algorithmic}
\end{algorithm}

\begin{algorithm}[h]
  \caption{Data Generation for Algorithm \ref{WGAN_GL}}
  \label{data_gen_proc_GL}
\begin{algorithmic}
  \STATE {\bfseries Input:} the seed drawn from the base measure $z\sim \mu_n$, the parameter $\tau$ and $\theta$ of the critic and the generator, the learning rate $\eta$,
  the number of gradient layers $l$.\\
   \vspace{1mm}
   \STATE Apply gradient layers $l$ times $z' \leftarrow G_\eta^{\tau,\theta}\circ \cdots \circ G_\eta^{\tau,\theta}(z)$\\
   \STATE Return the sample $g_{\theta}(z')$.
\end{algorithmic}
\end{algorithm}

We next briefly review Algorithm \ref{WGAN_GL} in which a fixed number of gradient layers with latest parameters is inserted in the bottom of a generator of WGAN-GP.
That is, gradient layers modify a noise distribution $\mu_n$ to improve the quality of a generator by the functional gradient method.

\section{Brief Review of Wasserstein Distance}
We introduce some facts concerning the Wasserstein distance, which is used for the proof of Proposition \ref{opt_wasserstein_prop}.
We first describe a primal form of the Wasserstein distance.
For $p\geq 1$ let $\mathcal{P}_p$ be the set of Borel probability measures with finite $p$-the moment on $\mathcal{X} \subset \mathbb{R}^v$.
For $\mu, \nu \in \mathcal{P}_p$ a probability measure $\gamma$ on $\mathcal{X} \times \mathcal{X}$
satisfying $\pi^1_\sharp \gamma = \mu$ and $\pi^2_\sharp \gamma = \nu$ is called a {\it plan} ({\it coupling}),
where $\pi^i$ denotes the projection from $\mathcal{X} \times \mathcal{X}$ to the $i$-th space $\mathcal{X}$.
We denote by $\Gamma(\mu,\nu)$ the set of all plans between $\mu$ and $\nu$.
We now introduce Kantorovich's formulation of the $p$-Wasserstein distance $W_p$ for $p\geq 1$.
\begin{equation}
  W_p^p(\mu,\nu) = \min_{\gamma \in \Gamma(\mu,\nu)} \int_{\mathcal{X}\times \mathcal{X}}\| x-y\|_2^p d\gamma(x,y) \label{kantrovich}
\end{equation}
When $p=1$ and $\mu,\nu$ have bounded supports, there is the Kantorovich-Rubinstein dual formulation of the $1$-Wasserstein distance, which coincide with the definition introduced in the paper.
The existence of optimal plans is guaranteed under more general integrand (c.f. \cite{villani2008optimal,ambrosio2008gradient}) and we denote by $\Gamma$ the set of optimal plans.
Prior to this formulation, the optimal transport problem in Monge's formulation was proposed.
\begin{equation}
  \inf_{\phi_\sharp \mu = \nu }  \int_{\mathcal{X}}\| x-\phi(x)\|_2^p d\mu(x), \label{monge}
\end{equation}
where the infimum is taken over all transport maps $\phi: \mathcal{X} \rightarrow \mathcal{X}$ from $\mu$ to $\nu$, i.e., $\phi_\sharp \mu = \nu$.
Because a transport map $\phi$ gives a plan $\gamma = (id\times \phi)_\sharp \mu$, we can easily find (\ref{kantrovich}) $\leq $ (\ref{monge}).
In general, an optimal transport map that solves the problem (\ref{monge}) does not always exist unlike Kantrovich problem (\ref{kantrovich}).
However, in the case where $p>1$, $\mathcal{X}=\mathbb{R}^v$, and $\mu$ is absolutely continuous with respect to the Lebesgue measure,
the existence of optimal transport maps is guaranteed \cite{brenier1987decomposition, brenier1991polar} and it is extended to more general integrand (see \cite{ambrosio2008gradient}).
Moreover, this optimal transport map also solves Kantrovich problem (\ref{kantrovich}), i.e., these two distances coincide.
On the other hand, in the case $p=1$, the existence of optimal transport maps is much more difficult, but it is shown in limited settings as follows.

\begin{proposition} [Sudakov \cite{sudakov1979geometric}, see also \cite{ambrosio2003lecture}] \label{sudakov_prop}
  Let $\mathcal{X}$ be a compact convex subset in $\mathbb{R}^v$ and assume that $\mu$ is absolutely continuous with respect to Lebesgue measure.
  Then, there exists an optimal transport map $\phi$ from $\mu$ to $\nu$ for the problem \ref{monge} with $p=1$.
  Moreover, if $\nu$ is also absolutely continuous with respect to Lebesgue measure, we can choose $\psi$ so that $\psi^{-1}$ is well defined $\mu_0$-a.e., and $\phi^{-1}_\sharp \nu = \mu$.
\end{proposition}

Under the same assumption in Proposition \ref{sudakov_prop}, it is known that two distances (\ref{monge}) and (\ref{kantrovich}) coincide \cite{ambrosio2003lecture}, that is,
the Kantrovich problem (\ref{kantrovich}) is solved by an optimal transport map.

\section{Proofs}
We here the give proof of Proposition \ref{surrogate_prop}.
\begin{proof} [ Proof of Proposition \ref{surrogate_prop}]
Note that $\mathcal{L}(\psi) = \hat{\mathcal{L}}(f_{\psi}^*,\psi)$.
For $\psi \in B_r^\infty(\phi)$, we divide $\mathcal{L}(\psi)$ into two terms as follows.
\begin{equation}
  \mathcal{L}(\psi) = (\hat{\mathcal{L}}(f_\psi^*,\psi) - \hat{\mathcal{L}}(f_\phi^*,\psi)) + \hat{\mathcal{L}}(f_\phi^*,\psi). \label{divide_obj}
\end{equation}
We first bound the first term in (\ref{divide_obj}) by $L$-smoothness of $\hat{\mathcal{L}}(f_{\psi'}^*,\psi)$ with respect to $\psi'$ at $\psi$ in $B_r^{\infty}(\psi)$.
\begin{equation*}
  \left| \hat{\mathcal{L}}(f_{\phi}^*,\psi) - ( \hat{\mathcal{L}}(f_{\psi}^*,\psi) + \pd< \nabla_{\psi'}\hat{\mathcal{L}}(f_{\psi'}^*,\psi)\big|_{\psi'=\psi}, \phi-\psi >_{L^2(\mu_g)}) \right| 
  \leq \frac{L}{2}\| \phi-\psi \|_{L^2(\mu_g)}^2.
\end{equation*}
Since $\hat{\mathcal{L}}(f_{\psi'}^*,\psi)$ attains the maximum, we have $\nabla_{\psi'}\hat{\mathcal{L}}(f_{\psi'}^*,\psi)\big|_{\psi'=\psi} = 0$ and have
\begin{equation}
  \left| \hat{\mathcal{L}}(f_{\phi}^*,\psi) - \hat{\mathcal{L}}(f_{\psi}^*,\psi)  \right| 
  \leq \frac{L}{2}\| \phi-\psi \|_{L^2(\mu_g)}^2.  \label{lipschitz_bound_1}
\end{equation}

We next bound $\hat{\mathcal{L}}(f_\phi^*,\psi)$ in (\ref{divide_obj}).
We remember that
\begin{equation}
  \hat{\mathcal{L}}(f_\phi^*,\psi) = \mathbb{E}_{x\sim\mu_D}[f_\phi^*(x)]-\mathbb{E}_{x \sim \mu_g}[f_\phi^*\circ \psi(x)] - \lambda R_{f_\phi^*}. \label{obj_def}
\end{equation}
By $L$-smoothness of $f_\phi^*$, it follows that
\begin{equation*}
  \left| f_\phi^*(\psi(x)) - ( f_\phi^*(\phi(x)) + \pd< \nabla_z f_\phi^*(z)\big|_{z=\phi(x)},\psi(x)-\phi(x)>_2 )\right| \leq \frac{L}{2}\|\psi(x)-\phi(x)\|_2^2.
\end{equation*}
By taking the expectation with respect to $\mathbb{E}_{\mu_g}$, we get
\begin{equation*}
  \left| -\mathbb{E}_{x \sim \mu_g}[f_\phi^*\circ \psi(x)] +  \mathbb{E}_{\mu_g}[f_\phi^*(\phi(x))] + \pd< \nabla_z f_\phi^*\circ \phi,\psi-\phi>_{L^2(\mu_g)} \right|
  \leq \frac{L}{2}\|\psi-\phi\|_{L^2(\mu_g)}^2. 
\end{equation*}

We substitute this inequality into (\ref{obj_def}), we have 
\begin{align}
  \hat{\mathcal{L}}(f_\phi^*,\psi) &\leq \mathbb{E}_{x\sim\mu_D}[f_\phi^*(x)] + \frac{L}{2}\|\psi-\phi\|_{L^2(\mu_g)}^2 \notag
  - ( \mathbb{E}_{\mu_g}[f_\phi^*(\phi(x))] + \pd< \nabla_z f_\phi^*\circ \phi,\psi-\phi>_{L^2(\mu_g)} ) - \lambda R_{f_\phi^*} \notag \\
  &= \hat{\mathcal{L}}(f_\phi^*,\phi) - \pd< \nabla_z f_\phi^*\circ \phi,\psi-\phi>_{L^2(\mu_g)} + \frac{L}{2}\|\psi-\phi\|_{L^2(\mu_g)}^2  \notag\\
  &= \mathcal{L}(\phi) + \pd< \nabla_\phi \mathcal{L}(\phi),\psi-\phi>_{L^2(\mu_g)} + \frac{L}{2}\|\psi-\phi\|_{L^2(\mu_g)}^2, \label{lipschitz_bound_2}
\end{align}
and the opposite inequality
\begin{equation}
  \hat{\mathcal{L}}(f_\phi^*,\psi) \geq \mathcal{L}(\phi) + \pd< \nabla_\phi \mathcal{L}(\phi),\psi-\phi>_{L^2(\mu_g)} - \frac{L}{2}\|\psi-\phi\|_{L^2(\mu_g)}^2, \label{lipschitz_bound_2b}
\end{equation}
where we used $\nabla_\phi \mathcal{L}(\phi) = - \nabla_z f_\phi^*(z)|_{z=\phi(\cdot)}$.
By combining (\ref{divide_obj}),(\ref{lipschitz_bound_1}), and (\ref{lipschitz_bound_2}), we have

\begin{equation*}
  \mathcal{L}(\psi) \leq \mathcal{L}(\phi) + \pd< \nabla_\phi \mathcal{L}(\phi),\psi-\phi>_{L^2(\mu_g)} + L\| \phi-\psi \|_{L^2(\mu_g)}^2.
\end{equation*}

Moreover, since $\hat{\mathcal{L}}(f_\psi^*,\psi) - \hat{\mathcal{L}}(f_\phi^*,\psi) \geq 0 $ in (\ref{divide_obj}), we have $\mathcal{L}(\psi) \geq \hat{\mathcal{L}}(f_\phi^*,\psi)$.
Therefore, we get the opposite inequality by (\ref{lipschitz_bound_2b})
\begin{equation*}
    \mathcal{L}(\psi) \geq \mathcal{L}(\phi) + \pd< \nabla_\phi \mathcal{L}(\phi),\psi-\phi>_{L^2(\mu_g)} - \frac{L}{2}\| \phi-\psi \|_{L^2(\mu_g)}^2.
\end{equation*}
This finishes the proof.
\end{proof}

We next provide the proof of Theorem \ref{convergence_theorem}.
\begin{proof}[ Proof of Theorem \ref{convergence_theorem}]
  Noting that $\| \eta \nabla_{\phi_k} \mathcal{L}(\phi_k)\|_\infty \leq r$ and Lipschitz smoothness of $\mathcal{L}$, we have 
  \begin{align*}
    \mathcal{L}(\phi_{k+1}) &\leq \mathcal{L}(\phi_k) - \eta  \| \nabla_\phi \mathcal{L}(\phi_k)\|^2_{L^2(\mu_g)} + \frac{\eta^2L}{2}\|\nabla_\phi \mathcal{L}(\phi_k)\|_{L^2(\mu_g)} \\
    &= \mathcal{L}(\phi_k) - \eta(1-\eta L/2)  \| \nabla_\phi \mathcal{L}(\phi_k)\|^2_{L^2(\mu_g)}.
  \end{align*}

  Since $\eta \leq 1/L$, we have $\mathcal{L}(\phi_{k+1}) \leq \mathcal{L}(\phi_k) - \frac{\eta}{2}\| \nabla_\phi \mathcal{L}(\phi_k) \|^2_{L^2(\mu_g)}$.
  Summing up over $k\in \{0,\ldots,T-1\}$ and dividing by $T$ we obtain
  \begin{align*}
    \frac{1}{T} \sum_{k=0}^{T-1} \| \nabla_\phi \mathcal{L}(\phi_k) \|^2_{L^2(\mu_g)} \leq \frac{2}{\eta T}(\mathcal{L}(\phi_0)-\mathcal{L}(\phi_T)).
  \end{align*}
  This inequality finishes the proof of the theorem.
\end{proof}

\begin{proof} [ Proof of Proposition \ref{opt_wasserstein_prop} ]
  By Proposition \ref{sudakov_prop}, there exists an optimal transport $\psi$ from $\mu_g$ to $\mu_D$ and an optimal plan is given by $\gamma = (id\times \psi)_\sharp \mu_g$.
  We set $\psi_t = (1-t)id+t\psi$ and $\mu_t = \psi_{t\sharp}\mu_g$.
  Because $(\psi_s,\psi_t)_{\sharp}\mu_g$ ($0\leq s<t \leq 1$) gives a plan between $\mu_g$ and $\mu_D$, we have
  \begin{align}
    W_1(\mu_s,\mu_t) &\leq \int_{\mathcal{X}\times\mathcal{X}} \|x-y\|_2 d(\psi_s,\psi_t)_{\sharp}\mu_g \notag \\
    &= \int_{\mathcal{X}} \|\psi_s(x)-\psi_t(x)\|_2 d\mu_g \notag \\
    &= (t-s)\int_{\mathcal{X}} \|x-\psi(x)\|_2 d\mu_g = (t-s)W_1(\mu_g,\mu_D).   \label{geo_ineq_1}
  \end{align}
  We next prove the opposite inequality.
  Noting that $(id,\psi_s)_\sharp \mu_g$ is a plan from $\mu_g$ to $\mu_s$ and $(\psi_t,\psi)_\sharp \mu_g$ is a plan from $\mu_t$ to $\mu_D$, we have the following two inequalities
  \begin{align*}
    W_1(\mu_g,\mu_s) &\leq \int_{ \mathcal{X}\times \mathcal{X}} \|x-y\|_2 d(id,\psi_s)_{\sharp} \mu_g = \int_{ \mathcal{X} } \|x-\psi_s(x)\|_2 d \mu_g = s W_1(\mu_g,\mu_D), \\
    W_1(\mu_t,\mu_D) &\leq \int_{ \mathcal{X}\times \mathcal{X}} \|x-y\|_2 d(\psi_t,\psi)_{\sharp} \mu_g = \int_{ \mathcal{X} } \|\psi_t(x)-\psi(x)\|_2 d \mu_g = (1-t) W_1(\mu_g,\mu_D).
  \end{align*}
  Using these two inequalities and the triangle inequality, we get
  \begin{equation*}
    W_1(\mu_g,\mu_D) \leq W_1(\mu_g,\mu_s) + W_1(\mu_s,\mu_t) + W_1(\mu_t,\mu_D) \leq (1+s-t)W_1(\mu_g,\mu_D) +  W_1(\mu_s,\mu_t). 
  \end{equation*}
  That is $(t-s)W_1(\mu_g,\mu_D) \leq W_1(\mu_s,\mu_t)$.
  By combining this inequality and (\ref{geo_ineq_1}), we have $(t-s)W_1(\mu_g,\mu_D) = W_1(\mu_s,\mu_t)$ and this finishes the proof.
\end{proof}

\section{Labeled Faces in the Wild}
In this section we provide the result on the Labeled Faces in the Wild dataset.
The result is shown in Figure \ref{lfw_generated}.
After training WGAN-GP (left), we ran Algorithm \ref{FGD} for a few iterations (right).
\begin{figure}[th]
 \includegraphics[width=78mm,angle=0]{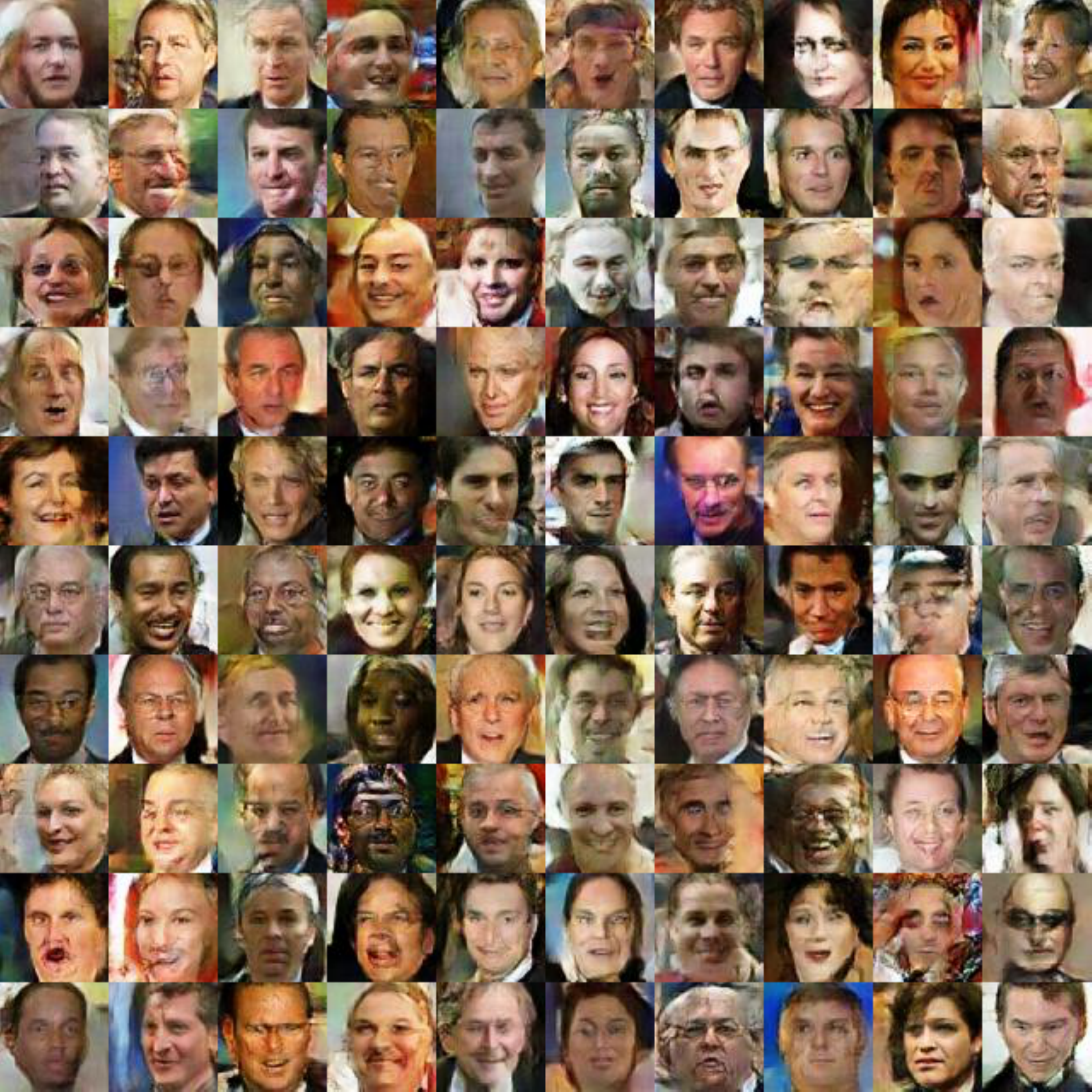}
 \includegraphics[width=78mm,angle=0]{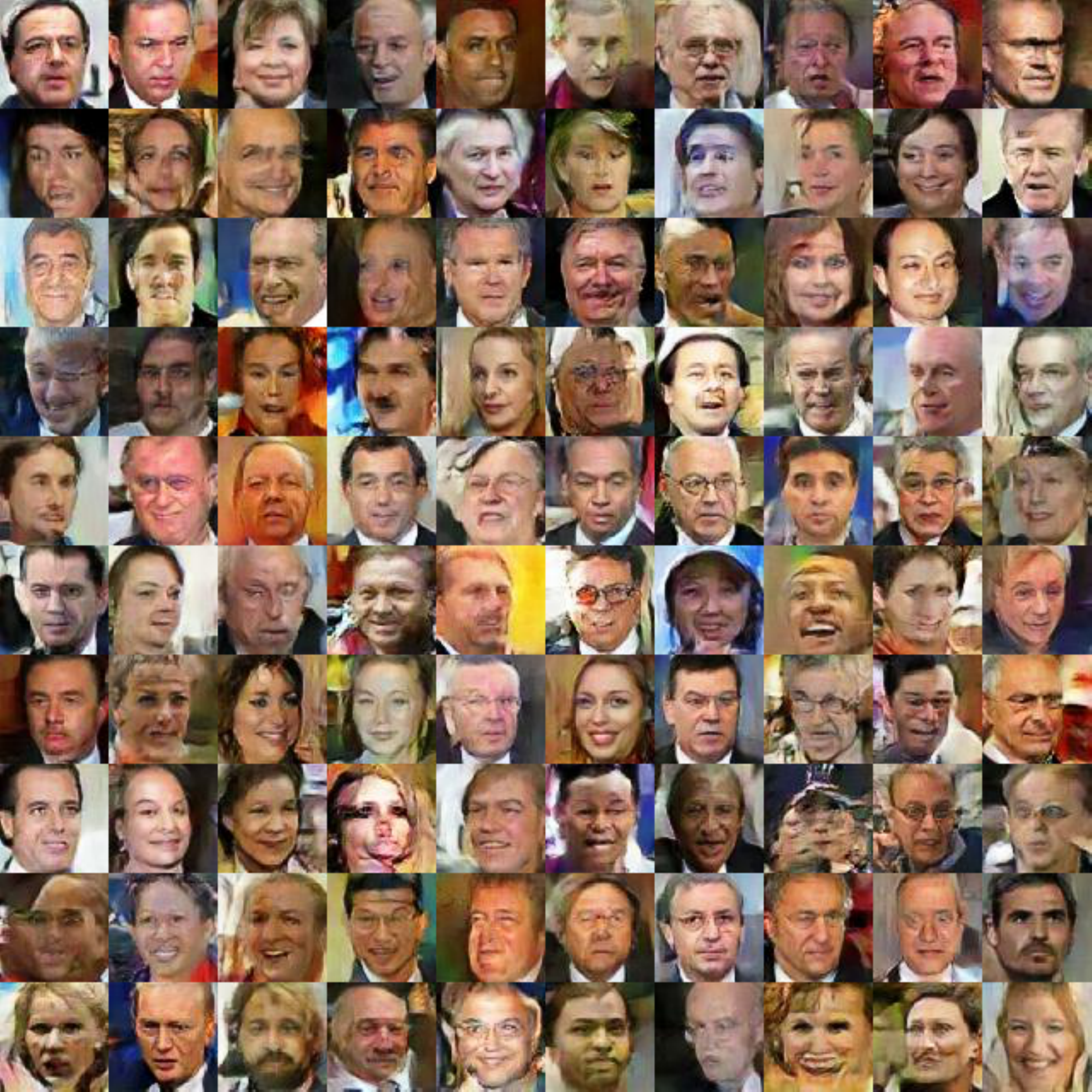}
 \caption{Random samples drawn from the generator trained by WGAN-GP (left) and the gradient layer (right).} \label{lfw_generated}
\end{figure}

\end{document}